\documentclass{article}

\newif\ifreview 
\newif\ifarxiv \newcommand{\arxiv}{\arxivtrue}
\newif\ifcamera 
\newif\ifrebuttal 

\PassOptionsToPackage{numbers, compress}{natbib}

\arxiv

\ifreview \usepackage[main]{neurips_2026} \fi
\ifarxiv \usepackage[preprint]{neurips_2026} \fi
\ifrebuttal \usepackage[main]{neurips_2026} \fi
\ifcamera \usepackage[main, final]{neurips_2026} \fi

\usepackage[utf8]{inputenc} 
\usepackage[T1]{fontenc}    
\usepackage{hyperref}       
\usepackage{url}            
\usepackage{booktabs}       
\usepackage{amsfonts}       
\usepackage{nicefrac}       
\usepackage{microtype}      
\usepackage{xcolor}         

\def\paperTitle{
Turbo-Muon: Almost-Orthogonal Pre-Conditioning for Fast Muon Updates
}

\def\authorBlock{
\hspace{0cm}
\textbf{Thibaut Boissin}$^{1,2,\dagger}$
\hspace{1mm}
\textbf{Thomas Massena}$^{2,3,\dagger}$
\hspace{1mm}
\textbf{Franck Mamalet}$^{1}$
\hspace{1mm}
\textbf{Mathieu Serrurier}$^{2}$
\vspace{3mm}
\\
$^1$ Institut de Recherche Technologique Saint-Exupery, France\\
$^2$ IRIT, France\\
$^3$ Innovation \& Research Division, SNCF, France\\
}

\usepackage{graphicx}	
\usepackage{amsmath}	
\usepackage{amssymb}	
\usepackage{amsthm}
\usepackage{booktabs}
\usepackage{times}
\usepackage{microtype}
\usepackage{epsfig}
\usepackage{caption}
\usepackage{float}
\usepackage{array}
\usepackage{color, colortbl}
\makeatletter
\@ifundefined{insert@pcolumn}{\let\insert@pcolumn\insert@column}{}
\makeatother
\usepackage{stfloats}
\usepackage{enumitem}

\usepackage{tabularx}
\usepackage{xstring}
\usepackage{multirow}
\usepackage{xspace}
\usepackage{url}
\usepackage{subcaption}
\usepackage{xcolor}
\usepackage[hang,flushmargin]{footmisc}
\usepackage{bbm}
\usepackage{ulem}

\newtheorem{proposition}{Proposition}

\newtheorem{lemma}{Lemma}

\usepackage{xr-hyper}
\usepackage{makecell}

\usepackage{wrapfig}

\ifcamera \usepackage[accsupp]{axessibility} \fi

% ...

\definecolor{muoncolor}{HTML}{0072B2}
\definecolor{dioncolor}{HTML}{56B4E9}
\definecolor{turbomuoncolor}{HTML}{009E73}
\definecolor{frobcolor}{HTML}{7E7E7E}

\ifarxiv  \fi

\newcommand{\R}[1]{{%
    \textbf{%
        \ifstrequal{#1}{1}{\textcolor{red}{R#1}}{%
        \ifstrequal{#1}{2}{\textcolor{blue}{R#1}}{%
        \ifstrequal{#1}{3}{\textcolor{magenta}{R#1}}{%
        \ifstrequal{#1}{4}{\textcolor{teal}{R#1}}{%
                           \textcolor{cyan}{R#1}%
        }}}}%
    }%
}}

\newcommand{\turbomuon}{Turbo-Muon }

\newif\ifshowcomments
\showcommentstrue 

\ifshowcomments

\newcommand{\thomas}[1]{{\color{blue}Thomas:#1}}
\newcommand{\thib}[1]{{\color{magenta}Thib:#1}}
\newcommand{\franck}[1]{{\color{red}Franck:#1}}
\else
\newcommand{\thib}[1]{}
\newcommand{\thomas}[1]{}
\newcommand{\franck}[1]{{}
\fi

\makeatletter
\newcommand*{\addFileDependency}[1]{
  \typeout{(#1)}
  \@addtofilelist{#1}
  \IfFileExists{#1}{}{\typeout{No file #1.}}
}

\makeatother
\newcommand*{\myexternaldocument}[1]{
    \externaldocument{#1}
    \addFileDependency{#1.tex}
    \addFileDependency{#1.aux}
}

\usepackage[capitalize]{cleveref}
\crefname{section}{Sec.}{Secs.}
\crefname{table}{Table}{Tables}
\crefname{figure}{Fig.}{Figs.}
\crefname{algorithm}{Alg.}{Algs.}

\ifarxiv \crefname{appendix}{App.}{Apps.}
\else \crefname{appendix}{Suppl.}{Suppls.} \fi

\usepackage{algorithm}%
\usepackage{algorithmicx}%
\usepackage{algpseudocode}%

\newtheorem*{lemma*}{Lemma}

\title{\paperTitle}
\author{\authorBlock}

\begin{document}
\maketitle

\ifarxiv
    \begingroup
    \renewcommand\thefootnote{\fnsymbol{footnote}}
    \footnotetext[2]{Core contributors.}
    \endgroup
\fi

\begin{abstract}
Orthogonality-based optimizers, such as Muon, have recently shown strong performance across large-scale training and community-driven efficiency challenges. However, these methods rely on a costly gradient orthogonalization step. Even efficient iterative approximations such as Newton-Schulz remain expensive, typically requiring dozens of matrix multiplications to converge.

We introduce a pre-conditioning procedure that improves the initialization of the Newton–Schulz iterations while incurring negligible overhead.
Furthermore, our pre-conditioning reduces the initial polar error and enables the removal of one Newton-Schulz iteration (out of the five iterations usually used in practice). 
The resulting implementation significantly reduces Muon’s overhead.

At the end-to-end training level, we observe consistent runtime improvements across speed-run and standard benchmarks, including $\sim3\%$ reductions in training time
 on multiple fast training benchmarks, while matching reference performance on both language and vision tasks. Crucially, these improvements require no hyperparameter tuning and can be adopted as a simple drop-in replacement.

Beyond empirical gains, we provide theoretical insight into the geometry of the update and its potential robustness against feature collapse.
Our code is publicly available on 
\ifarxiv \href{https://github.com/thib-s/flash-newton-schulz/}{github}, in \href{https://optax.readthedocs.io/en/latest/api/generated/optax.contrib.muon.html}{optax} and \href{https://huggingface.co/tboissin/newton_schulz_triton}{huggingface kernels}.
\else \href{https://anonymous.4open.science/r/flash-newton-schulz-D67B/}{anonymous-github}.
\fi
\end{abstract}


\section{Introduction}

Optimization algorithms are a central component of modern deep learning, directly affecting training efficiency, stability, and final model performance. Most widely used methods build on stochastic gradient descent, either directly through momentum-based variants or through adaptive extensions such as Adam and AdamW~\cite{kingma2014adam,loshchilovdecoupled}. More recently, orthogonalization of weight updates has become a central ingredient in several optimizers~\cite{bernstein2024oldoptimizernewnorm, bernstein2024modulardualitydeeplearning}.
The most prominent example is the Muon optimizer~\cite{jordan2024muon}, which has been shown to consistently surpass AdamW~\cite{kingma2014adam,loshchilovdecoupled} across diverse training regimes~\cite{wen_fantastic_2025} and has been adopted in large foundation models such as Kimi-K2 and GLM-4.5~\cite{kimi2025k2,zeng2025glm}. 
Recent large-scale evaluations further report favorable scaling of Muon for LLM training~\cite{liu2025muonscalablellmtraining, shah2025practical}. 
In Muon and its variants, updates are projected toward the orthogonal manifold to
 accelerate convergence, stabilize training, and support robust hyperparameter transfer across model scales~\cite{bernstein2025deriving, bernstein2024modulardualitydeeplearning, large2024scalable, pethick2025training}.

However, the computational cost of this projection remains the primary barrier to broad adoption at scale. Exact orthogonalization via SVD, while numerically precise, is impractical on modern accelerators due to its high cost and instability for large matrices in low precision.
In practice, state-of-the-art methods therefore rely on efficient but imprecise iterative schemes such as the Newton-Schulz method (abbreviated NS and also known as Bj\"orck)~\cite{bjorck1971iterative, cesista2025muonoptcoeffs}.

In these methods, wall-clock overhead scales with the number of iterative steps needed to achieve a target polar error (i.e., the Frobenius distance to the closest orthogonal matrix), creating a direct 
trade-off 
between runtime and optimization quality~\cite{shulgin2025beyond}. 
This explains why Newton-Schulz is usually used with only five iterations. This work focuses on improving this specific, low-iteration regime, which is most relevant to deep-learning practitioners.

\textbf{This paper.} 
We introduce a pre-conditioned Newton-Schulz method based on Almost Orthogonal Layer (AOL) pre-conditioning~\cite{prach2022almost}. We show that this pre-conditioning yields a significantly improved initialization of the Newton-Schulz iteration toward the polar factor, enabling the removal of one Newton-Schulz iteration. The pre-conditioning incurs negligible overhead by reusing existing computations, leading to a better trade-off between runtime and orthogonalization quality (\cref{fig:pareto8192}).
Crucially, we prove that the resulting update remains a steepest descent direction under a modified norm.

Empirically, \turbomuon acts as a drop-in replacement for Muon~\cite{jordan2024muon}, consistently reducing training time by up to 4\% while preserving downstream performance across diverse settings, including highly optimized speed-run benchmarks on NanoGPT and CIFAR-10 as well as the large-scale ImageNet-1K benchmark.

\section{Related Work}
\label{sec:related}

\paragraph{Orthogonality-based optimizers} form a recent class of training algorithms that enforce or approximate orthogonal updates to improve conditioning and stability during deep network optimization. The foundational work of Muon \cite{jordan2024muon} introduced the idea of orthogonalizing parameter updates via iterative matrix normalization, such as the Newton-Schulz method, yielding isotropic update directions and smoother convergence dynamics. Follow-up works such as Deriving Muon \cite{bernstein2025deriving} and Old Optimizer, New Norm \cite{bernstein2024oldoptimizernewnorm} offered theoretical insights linking these methods to optimization under alternative geometries and norms. Scalable implementations such as Dion \cite{ahn_dion_2025} extended the approach to distributed settings, while Gluon \cite{riabinin_gluon_2025} and AdaMuon \cite{si_adamuon_2025} incorporated adaptivity and layer-wise refinements, bridging the gap between theoretical linear minimization oracle (LMO) frameworks and practical large-model training~\cite{pethick2025training}. Empirical studies \cite{modded_nanogpt_2024, wen_fantastic_2025, kimi2025k2} further demonstrated that such orthogonalization can accelerate convergence, stabilize training under heavy-tailed gradients, and enhance performance across architectures from GPT-2 scale to billion-parameter models. Together, these works position orthogonality-based optimization as a promising direction for efficient and stable large-scale learning.

\paragraph{Orthogonalization methods.} 
A variety of orthogonalization schemes have been proposed to construct weight matrices with orthogonal constraints.
The Modified Gram-Schmidt QR factorization~\citep{LaPlace1820} finds the polar factor with an iterative process (one iteration per row). 
Alternatively, the Cayley transform~\citep{Cayley_1846} establishes a bijection between skew-symmetric and orthogonal matrices through $(I-A)(I+A)^{-1}$, but requires a costly matrix inversion.
The exponential map~\citep{singla_skew_2021} also leverages skew-symmetric matrices, generating $Q=\exp(A)$ while typically approximating the exponential via truncated series. 
Similarly, The Cholesky-based method~\citep{hu_recipe_2023} orthogonalizes a matrix with a triangular decomposition $MM^\top=L L^\top$ (where $L$ is triangular) and then solves for $L^{-1}M$, offering efficiency when numerical stability is maintained. 
Finally, the iterative Newton-Schulz algorithm~\citep{bjorck1971iterative,anil_sorting_2019} (also known as the Bj\"orck--Bowie algorithm) projects matrices toward the Stiefel manifold via the computation of matrix polynomials.
This iteration achieves fast convergence when the input matrix is spectrally normalized.
Recent works~\cite{jordan2024muon, cesista2025muonoptcoeffs} extended this algorithm with fifth-order polynomials, achieving faster convergence.
While these methods provide complementary trade-offs between accuracy, stability, and computational cost, Newton-Schulz variants have gained substantial popularity due to their scalability, making these operations feasible even at large scales~\cite {kimi2025k2}.

\section{Background and Motivation}
\label{sec:bgnd_motiv}

\begin{figure*}[t]
    \centering
    \begin{subfigure}{0.45\textwidth}
        \centering
        \includegraphics[width=0.99\linewidth]{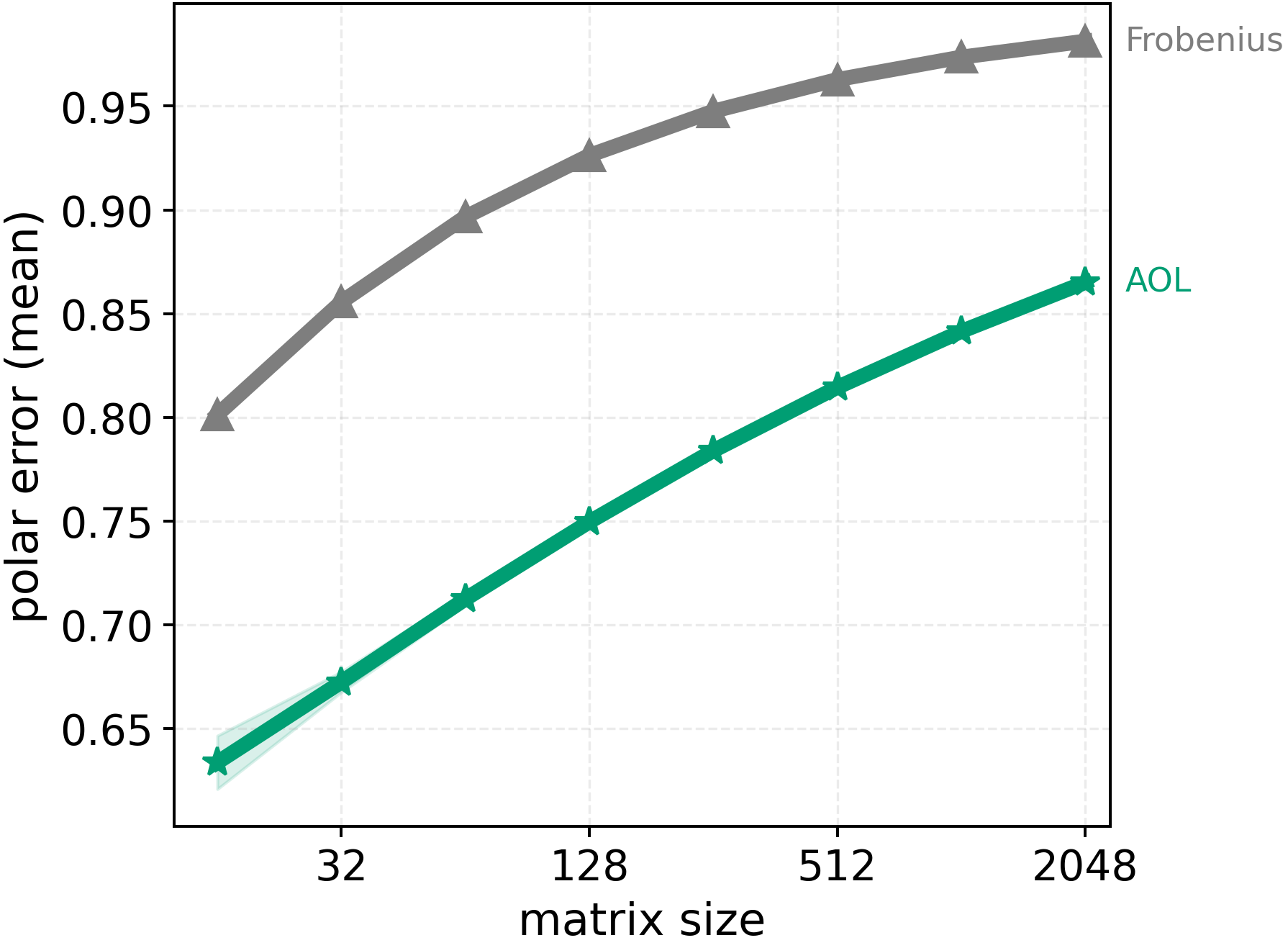}
        \caption{\textbf{Comparing pre-conditioning methods:} While normalization is primarily used for numerical stability, AOL also reduces the polar error of the initial iterate of Newton-Schulz.
        }
        \label{fig:precond_bench}
        \vspace{4mm}
    \end{subfigure}
    \hfill
    \begin{subfigure}{0.49\textwidth}
        \centering
        \includegraphics[width=0.99\linewidth]{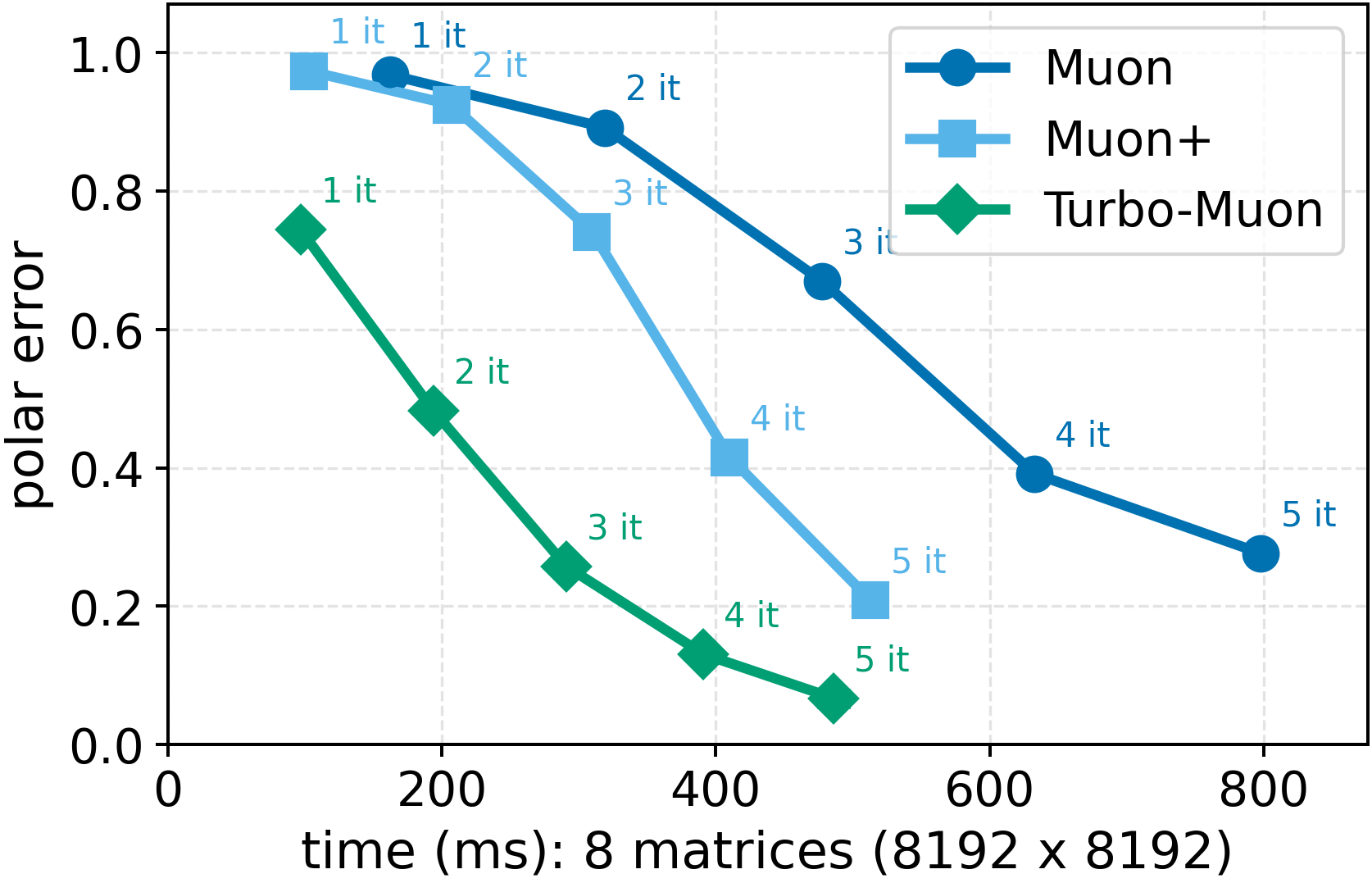}
        \caption{\textbf{Practical implementations of orthogonalization face a trade-off between polar error and computation time}. Thanks to pre-conditioning, our method reduces the polar error of the Newton-Schulz algorithm. This improves the trade-off between convergence and runtime.}
        \label{fig:pareto8192}
    \end{subfigure}
    \caption{\textbf{Our approach} consists of accelerating the core computation of the Muon optimizer: the Newton-Schulz algorithm. With adequate pre-conditioning (\ref{fig:precond_bench}), this routine can achieve lower polar error in less time, translating into a more efficient optimization, either in terms of performance (fewer steps) or runtime (lower time per step).}
\end{figure*}

\paragraph{Definitions.}
    Let $\mathbf{X} \in \mathbb{R}^{n \times n}$ denote a square matrix (all results generalize to the non-square case). 
    The singular value decomposition (SVD) of $\mathbf{X}$ is given by
    \[
        \mathbf{X} = \mathbf{U} \boldsymbol{\Sigma} \mathbf{V}^\top,
    \]
    where $\mathbf{U}, \mathbf{V} \in \mathbb{R}^{n \times n}$ are orthogonal matrices and 
    $\boldsymbol{\Sigma} = \mathrm{diag}(\sigma_1, \ldots, \sigma_n)$ is a diagonal matrix containing the singular values 
    $\sigma_1 \ge \cdots \ge \sigma_n \ge 0$. This also defines the spectral norm of a matrix: $ \| X \|_2 = \sigma_1$.
    
    A matrix $\mathbf{X}$ is said to be \textit{orthogonal} if it satisfies
        $\mathbf{X}^\top \mathbf{X} = \mathbf{X} \mathbf{X}^\top = \mathbf{I}$.
    For non-square matrices, only one of these two equalities can hold, in which case the matrix is referred to as 
    \textit{semi-orthogonal}.  
    To 
    assess
    deviations from orthogonality, we define the \textit{orthogonality error} as
        $\varepsilon_{\text{ortho}}(\mathbf{X}) = \|\mathbf{X}^\top \mathbf{X} - \mathbf{I}\|_F$.
    
    In the context of this paper, we seek the closest orthogonal matrix to a given matrix $\mathbf{X}$, 
    with respect to the Frobenius norm. 
    This matrix is the \textit{polar factor} ($\text{PolarFactor} (\mathbf{X})$), and, for any other matrix $Z$, we assess the approximation quality by the \textit{polar error} ($\varepsilon_{\text{polar}}(Z, \mathbf{X})$):
    \[
        \text{PolarFactor} (\mathbf{X}) = \mathbf{Q} = \mathbf{U}\mathbf{V}^\top \quad \text{and} \quad  \varepsilon_{\text{polar}}(Z, \mathbf{X}) = 
        \frac{\|Z  - \mathbf{Q}\|_F}{\sqrt{n}}.
    \]
    These metrics will be used throughout the paper to evaluate the orthogonalization accuracy of the $t^{th}$ iteration $\operatorname{NS}_t(\mathbf{X})$ of the Newton-Schulz algorithm.

\paragraph{Existing methods: strengths and limitations.}
The original \textit{Muon} formulation introduced the Newton-Schulz (NS) variant of the Björck iteration, which relies on quintic polynomial expansions to approximate the polar factor efficiently~\cite{jordan2024muon}. Building on this, \cite{cesista2025muonoptcoeffs} proposed iteration-dependent polynomial coefficients, allowing convergence within five or six steps--a significant improvement over the fixed-coefficient scheme. This idea was later extended and accelerated by~\cite{grishina_accelerating_2025} and \cite{amsel2025polar}, who analyzed optimal polynomial families and convergence regimes for orthogonalization.

In parallel, \cite{ahn_dion_2025} introduced \textit{Dion}, a distributed implementation supporting model sharding (i.e., splitting a large neural network model into smaller parts) and tensor-parallel decomposition. In addition to their optimizer, they provide an efficient implementation of Muon featuring Triton kernels~\cite{tillet2019triton} that efficiently exploit the matrix’s symmetric structure to reduce redundant computation and memory access. While these contributions collectively enhance the scalability of orthogonality-based optimizers, the computational cost of orthogonalization remains a limiting factor. 
In practice, Newton-Schulz iterations require dense matrix–matrix products, which interacts nontrivially with sharded training, leading to compute and communication bottlenecks in extreme-scale settings~\cite{essential2025muon}. Hence, each reduction in cost directly broadens the applicability of orthogonality-based optimizers to increasingly large models.

\paragraph{Newton-Schulz iteration}

To circumvent the limitations of exact methods, practical applications often rely on approximate methods. Notably, \cite{bjorck1971iterative} demonstrated that the following scheme:
\begin{align}
\label{eq:bjork}
    X_{k+1} = (1- \beta) X_k -  \beta X_k X_k^\top X_k
\end{align}
converges if $\|X_0\|_2\leq 1$ (in spectral norm) for $\beta \in \left[0, 0.5\right]$. More recently, \cite{jordan2024muon} introduced a variant that uses $5^{th}$-order polynomials:
\begin{align}
\label{eq:ns_quintic}
    X_{k+1} = a X_k + b X_k X_k^\top X_k + c X_k X_k^\top X_k X_k^\top X_k
\end{align}
which is computed in practice in three steps described in \cref{alg:muon} lines 3, 4, and 5.
This formulation opens the door to new optimizations: \cite{cesista2025muonoptcoeffs} showed that $(a,b,c)$ coefficients can be changed across iterations, using $(a_k, b_k, c_k)$ instead of a constant $(a,b,c)$. \cite{amsel2025polar} and \cite{grishina_accelerating_2025} both explored the convergence of these algorithms alongside methods to find efficient polynomial factors. Notably, they observed that the most suitable parameters rely on assumptions about the smallest singular values of $X$.
Throughout this paper and in our experiments, we compare to two widely used variants of the Newton-Schulz algorithm\footnote{Our experiments used the code from \href{https://github.com/KellerJordan/Muon}{github.com/KellerJordan/Muon} and \href{https://github.com/microsoft/dion}{github.com/microsoft/dion} for baseline evaluation.}:

\begin{itemize}
    \item \textbf{Muon}~\cite{jordan2024muon}: the original Newton-Schulz implementation, written in plain PyTorch. It typically uses five iterations, with constant polynomial factors.
    \item \textbf{Muon+}~\cite{ahn_dion_2025}: a more efficient implementation, integrating Triton kernels~\cite{tillet2019triton} proposed in \cite{ahn_dion_2025} and adaptive polynomial factors from~\cite{cesista2025muonoptcoeffs}, computed for five iterations. \footnote{Dion is a distributed variant of Muon (with non-trivial modifications), we only compare here to their implementation of the Newton-Schulz algorithm to keep results comparable.}
    \item \textbf{\turbomuon}: our approach, combining a pre-conditioning method with an additional fused Triton kernel to further reduce runtime and maintain high numerical precision. \turbomuon uses four iterations by default, and uses polynomial factors inherited from Muon+.
\end{itemize}
This selection allows us to observe the impact of theoretical improvements (such as adaptive polynomial factors) and design choices (such as Triton kernels). Other approaches such as \cite{grishina_accelerating_2025, amsel2025polar} are discussed in \cref{ap:polar_params}.

\section{Methodology
}

\begin{figure*}[t]
    \begin{minipage}[t]{0.45\textwidth}
        \begin{algorithm}[H]
            \caption{Newton-Schulz \color{muoncolor}Muon\color{black}/\color{dioncolor}Muon+\color{black}}
            \label{alg:muon}
            \begin{algorithmic}[1]
    \Require Initial point: $\mathbf{X_0} \in \mathbb{R}^{n\times n},(a,b,c)$  NS factors 
    \Statex
    \State \color{frobcolor}$s=1/\|\mathbf{X_0}\|_F$\color{black} \Comment{Frobenius scaling factor}
    \State $\mathbf{X_1}=\mathbf{X_0}s$  \Comment{k=1}
    \State $\mathbf{A_k}=\mathbf{X_k}^\top \mathbf{X_k} $\Comment{polynomial 1}
    \State $\mathbf{B_k}=b\mathbf{A_k}+ c \mathbf{A_k} \mathbf{A_k}$\Comment{polynomial 2}
    \State $\mathbf{X_{k+1}}=a\mathbf{X_k}+ \mathbf{X_k}\mathbf{B_k}$\Comment{polynomial 3}
    \Statex   ...\Comment {Iteration $k+1$ (no rescaling, step 3)} 
    \end{algorithmic}
        \end{algorithm}
    \end{minipage}
    \hfill
    \begin{minipage}[t]{0.54\textwidth}
        \begin{algorithm}[H]
            \caption{Newton-Schulz  \color{turbomuoncolor}\turbomuon\color{black}}
            \label{alg:turbomuon}
            \begin{algorithmic}[1]
    \Require Initial point: $\mathbf{X_0} \in \mathbb{R}^{n\times n},(a,b,c)$  NS factors
    \State $\mathbf{A_0}=  \mathbf{X_0}^\top \mathbf{X_0} $ \Comment{initial matmul}
    \State \color{turbomuoncolor}$s=(1/\sqrt{\|\mathbf{~|A_0}_i|~\|_1})_{i\in[0,n-1]}$\color{black}\Comment{AOL vector} 
    \State $\mathbf{X_1}=\mathbf{X_0}s$\Comment{rescaling $\mathbf{X_1}$} 
    \State \color{turbomuoncolor}$\mathbf{A_1}=s^\top \mathbf{A_0} s$\color{black} \Comment{$A_1$ without recomputation}
    \State $\mathbf{B_1}=b\mathbf{A_1}+ c \mathbf{A_1} \mathbf{A_1}$\Comment{polynomial 2}
    \State $\mathbf{X_2}=a\mathbf{X_1}+ \mathbf{X_1}\mathbf{B_1}$ \Comment{polynomial 3, k= 1}
    \Statex   ...\Comment {\color{muoncolor}Muon\color{black}~iteration $k+1$ (\cref{alg:muon} step 3)} 
            \end{algorithmic}
        \end{algorithm}
    \end{minipage}
    \caption{Our approach changes the normalization step in the first iteration of the Newton-Schulz algorithm. This allows for a faster convergence at a negligible cost since $A_0$ is reused in the first iteration. This approach does not impact memory consumption, which still maintains 3 buffers at most.}
\end{figure*}

Existing improvements of Newton-Schulz focus either on the polynomial steps applied~\cite{grishina_accelerating_2025, amsel2025polar}, or on efficient implementations of the algorithm~\cite{ahn_dion_2025}. However, all practical implementations start with an overlooked step: normalizing the input matrix by its Frobenius norm. This ensures that the Newton-Schulz iterative scheme converges~\cite{bjorck1971iterative}. In this work, we explore an alternative normalization scheme and observe significant improvements in the convergence of the first Newton-Schulz iterations. Moreover, this better performance allows the removal of a Newton-Schulz iteration without loss of precision.

\subsection{Almost Orthogonal Pre-conditioning}\label{ssec:AOLprecond}

To ensure $\|X_0\|_2\leq 1$, existing implementations perform an initial normalization step 
$X_0/\|X_0\|_F$
. While this guarantees convergence, in practice it often leads to $\|X_0\|_2\ll 1$ and does not modify the condition number, since the ratio $\sigma_{\max} / \sigma_{\min}$ remains unchanged.

In contrast, the ``Almost Orthogonal Layer" (AOL) parametrization introduced by \cite{prach2022almost} provides a lightweight alternative to normalization that simultaneously enforces $\|X_0\|_2\leq 1$ while improving the resulting matrix conditioning. In the following, we define: 

\begin{equation}
\text{AOL}(X_0) = X_0.\text{diag}\left(\sum_j |X_0^\top X_0|_{(i,j)}\right)^{-1/2} 
\label{eq:aol}
\end{equation}

where the inverse square root of the row sum of the Gram matrix acts as a scaling vector applied column-wise to $X_0$. This brings the matrix closer to an orthogonal form. In \cref{fig:precond_bench}, we compare the normalized polar error for the two normalization methods when applied to matrices sampled from a normal distribution (sizes ranging from 16 to 2048, with 32 matrices for each matrix size and algorithm).
As shown by~\citet{prach2022almost}, AOL converges to the polar factor if the matrix is near column-orthogonal and is of full rank. Importantly, this phenomenon can occur naturally on random high-dimensional matrices, as a byproduct of concentration phenomena which tend to enforce near-orthogonality (see Section 3.2.4 of \cite{vershynin2009high}). We also compare AOL with an alternative preconditioning strategy in \cref{ap:other_preconditioning}, further supporting its effectiveness in this setting.

\paragraph{Pre-conditioning for free.}
Standard implementations of the Newton-Schulz algorithm normalize the input matrix $X_0$ by its Frobenius norm up front to guarantee convergence (\cref{alg:muon} lines 1-2).
\turbomuon instead defers this normalization until after forming the Gram matrix $A_0 = X_0^\top X_0$ (\cref{alg:turbomuon} line 1). 
The AOL scaling vector $s$ is then computed directly from the row sums of $|A_0|$ via a cheap element-wise inverse square root (line 2).
The initialization of the Newton-Schulz states is performed by
applying this scaling vector via broadcasting ($A_1 = s^\top A_0 s$ and $X_1 = X_0 s$, \cref{alg:turbomuon} lines 3-4). 
Crucially, this reuses the Gram computation and avoids any additional matrix multiplications. AOL pre-conditioning thus incurs negligible runtime overhead, and only $\mathcal{O}(\min(m,n))$ extra memory.

\subsection{Benefits of AOL pre-conditioning for Newton-Schulz convergence}
\label{ssec:AOLxNS}

\begin{figure*}[t]
    \centering
    \begin{subfigure}{0.49\textwidth}
        \centering
        \includegraphics[width=\textwidth]{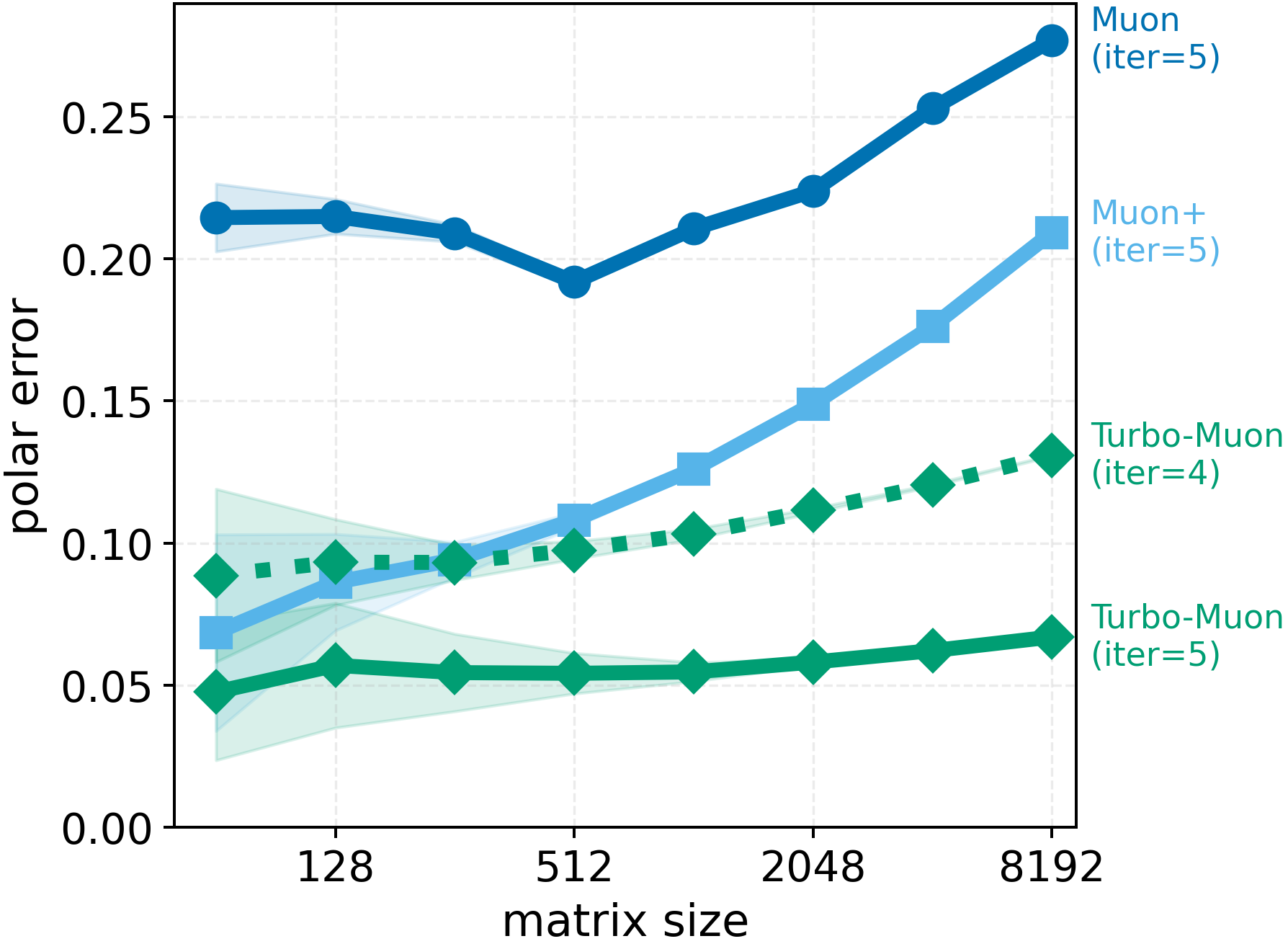}
        \caption{\textbf{Pre-conditioning is effective for large matrices:} Applying AOL before the algorithm improves convergence speed, especially for large matrices. In this context, an iteration can be removed while still achieving improved convergence.}
        \label{fig:polar_error_NS}
    \end{subfigure}
    \hfill
    \begin{subfigure}{0.49\textwidth}
        \centering
        \includegraphics[width=\textwidth]{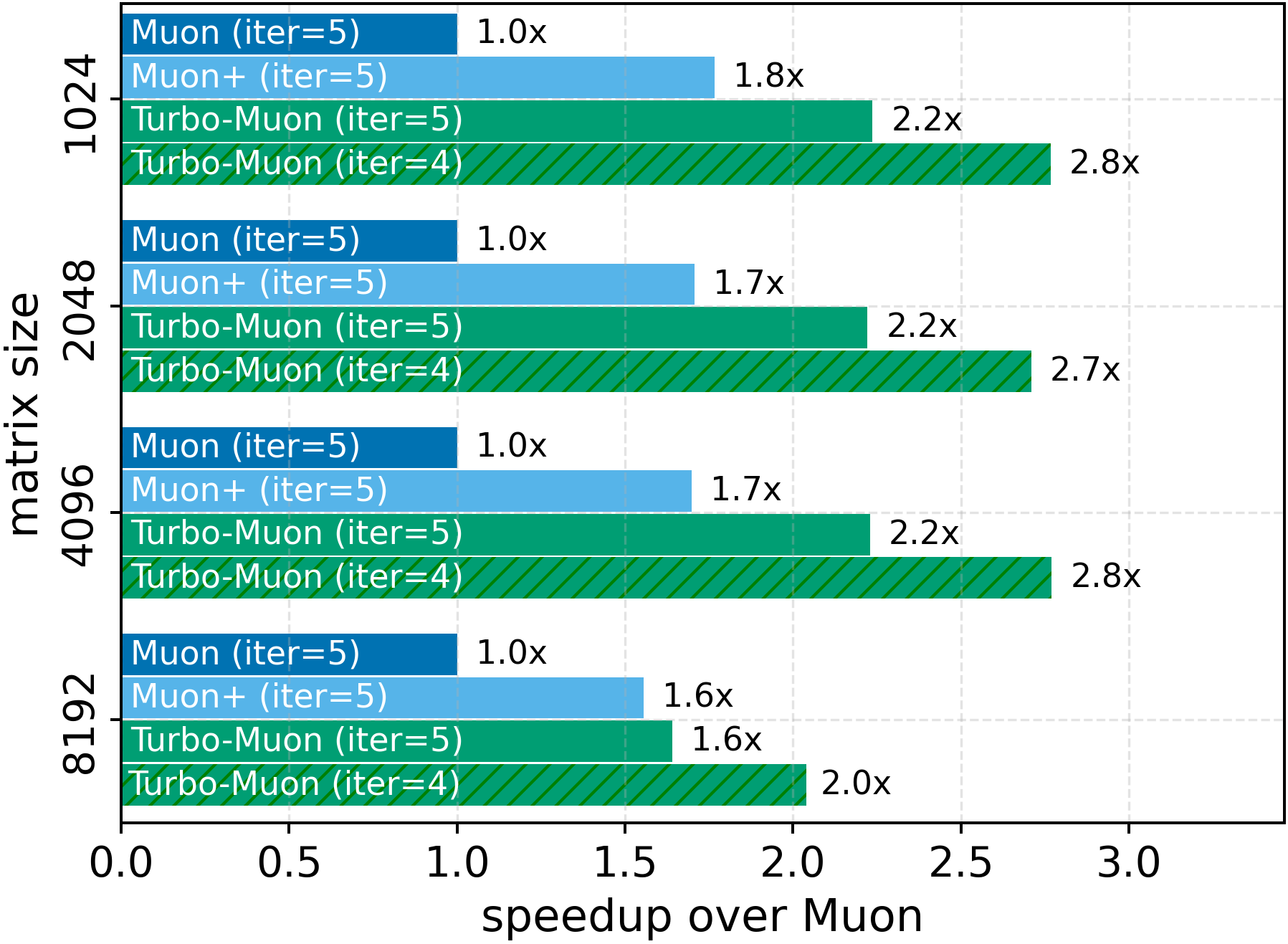}
        \caption{\textbf{Breakdown of our speedup:} Building atop Muon+, which includes Triton kernels and dynamic polynomial parameters ($1.7\times$ faster), our approach adds an extra Triton kernel that unlocks moderate gains ($2.2\times$ faster). Finally, removing one of the five iterations also improves runtime ($2.8\times$ faster).}
        \label{fig:runtime_NS}
    \end{subfigure}
    \caption{\textbf{Turning pre-conditioning into runtime:} Applying AOL before the algorithm improves its convergence speed (\textbf{a}). This can be used to remove an iteration while achieving a similar polar error. Removing one of the five iterations improves the runtime
    (\textbf{b}), making optimizers like Muon more scalable to large matrices.
    }
    \label{fig:pre-conditioning_time}
\end{figure*}

The observation made in \cref{ssec:AOLprecond} and \cref{fig:precond_bench} motivates the use of AOL as a pre-conditioner with the goal of removing one iteration from the Newton-Schulz scheme. Our objective is to determine whether iterative approaches to estimating the polar factor exhibit improved performance when preceded by AOL pre-conditioning in the Muon setting.
In this section, we conduct empirical studies on the effects of AOL pre-conditioning on the convergence of the \turbomuon algorithm toward the ideal polar factor. Our study shows that under widely adopted practical constraints (i.e., using five iterations), our approach yields better approximations of the polar factor on a variety of random matrices.

\paragraph{Adequate pre-conditioning reduces the polar error.} 
In \cref{fig:polar_error_NS}, for each matrix size, we sampled random matrices from a standard normal distribution\footnote{Experiments in \cref{ap:real_gradients} show that our results hold for realistic gradient distributions.} and compared the three  implementations listed in \cref{sec:bgnd_motiv} on the same samples.
We estimate confidence intervals on the polar error (with 32 samples) 
using the singular value decomposition applied to the non-conditioned matrix (as defined in \cref{sec:bgnd_motiv}). It first 
shows how the dynamic polynomial parameters from \cite{ahn_dion_2025} reduce the polar error from the baseline of \cite{jordan2024muon}. Additionally, AOL pre-conditioning (\textit{Turbo-Muon iter=5} in \cref{fig:polar_error_NS}) further reduces the polar error significantly. 
Importantly, our approach lowers the polar error enough to allow the removal of one iteration
(\textit{Turbo-Muon iter=4}), while still achieving a competitive or improved polar error. 

\paragraph{Removing one iteration and its impact on runtime.}
The default Newton-Schulz setting in Muon uses five iterations. Importantly, we have shown that we can leverage AOL pre-conditioning to remove a Newton-Schulz iteration without losing out on approximation quality relative to Muon and Muon+. Indeed, removing a single iteration reduces the total computational cost by $\approx 20\%$. This becomes particularly advantageous with large matrices, where each iteration follows a time complexity of $\mathcal{O}(n^3)$.
Alongside this cost reduction, we leverage the Triton kernels proposed in \cite{ahn_dion_2025, lin2025flash}, which improves computation time by using the symmetric nature of the operations used in 
\cref{alg:muon} lines 3 and 4:
since $X^\top X$ is symmetric, only one triangular part needs to be computed, effectively saving half of the required computation. However, 
\cref{alg:muon} line 5
uses a PyTorch operator that induces a duplicate load of $X$, while this has a small impact on the overall runtime, fixing this issue with a dedicated Triton kernel has a multiplicative effect with the removal of one iteration, leading to increased speedups.

\paragraph{Stress-testing on gradient-like matrices.} The convergence experiments above with normal distribution matrices may not fully reflect the statistics of gradients encountered during neural network training. In particular, empirical studies by~\citet{Simsekli2019TailIndex} have shown that stochastic gradients often exhibit heavy-tailed behavior, with tail indices typically ranging from $1.0$ to $1.8$ in classification networks, and \citep{kunstner2024heavy} shows similar phenomena in language transformers. 
In \cref{ap:real_gradients}, we conduct additional experiments on matrices sampled from a heavy-tailed Levy distribution, and show that the conclusions of our empirical convergence experiments remain valid in these settings.

\subsection{Ablation study}

Since the proposed method relies on pre-conditioning the gradient matrix, and experiments are conducted using the optimized version \color{dioncolor}{Muon+}\color{black}—which integrates Triton kernels~\cite{tillet2019triton} and adaptive polynomial factors from~\cite{cesista2025muonoptcoeffs}—we provide in \cref{tab:ablation_4096} an ablation study that separates the contributions of each component. The results show that the proposed pre-conditioning yields both a significant reduction in polar error (0.06 compared to 0.17) and a substantial runtime speedup (up to 2.83).

    \begin{table}[t]
        \centering
        \begin{tabular}{l l l l c}
            \toprule
            Variant & Source & Polar error & Runtime & Cumulative speedup\\
            \midrule
            \color{muoncolor}{Muon:}\color{black}~baseline  (5it) & \cite{jordan2024muon}& 0.25 & 130 ms & 1.0\\
            + Dynamic coeffs & \cite{cesista2025muonoptcoeffs}  & 0.17 & unchanged & 1.0\\
            \color{dioncolor}{Muon+:}\color{black}~+ Triton kernel& \cite{ahn_dion_2025} & unchanged & 77 ms  & 1.69 \\
            + AOL & Ours & 0.06 & unchanged $^\ast$ & 1.69\\ 
            + 3rd Triton kernel & Ours & unchanged & 59 ms & 2.20\\
            \color{turbomuoncolor}{\turbomuon:}\color{black}~-- 1 iter. (4it) & Ours & 0.12 & 46 ms & 2.83 \\
            \bottomrule
        \end{tabular}
        \vspace{3mm}
        \caption{\textbf{Ablation summary}. We report polar error and runtime improvement for each component, using the same protocol as \cref{fig:pareto8192} and \cref{fig:polar_error_NS} with matrix size set to  4096. $^\ast$AOL overhead is negligible for large matrices: matrix--vector multiplication is $4096\times$ cheaper than matrix--matrix multiplication.}
        \label{tab:ablation_4096}
    \end{table}

\section{Theoretical Analysis}
\label{sec:steepest}

Although the previous section shows that, in the practical regime of fewer than five iterations, the proposed method achieves lower polar error (relative to the true projection of the gradients) with reduced runtime, the introduced pre-conditioning modifies the asymptotic target of the Newton-Schulz iterations. In particular, as the number of iterations increases, the method converges to $\operatorname{PolarFactor}(AOL(G))$ rather than $\operatorname{PolarFactor}(G)$ (see \cref{ap:bias_analysis} for a detailed analysis of this residual error across iterations).

We build on the work of \citet{bernstein2024oldoptimizernewnorm}, who show that several optimizers can be interpreted as first-order steepest descent methods under a particular norm. Our analysis, detailed in \cref{ap:proof_convergence}, shows that AOL pre-conditioning leads to a problem close to a steepest descent under 
an induced norm determined by the AOL rescaling.
However, this norm is dependent on $S$, and its impact on the global optimization process remains unclear. Using only the general properties of the matrix $S$, we will answer the following question: if pre-conditioning induces the steepest descent in a different norm, can this bias the global optimization process?

        \begin{lemma}
        \label{lem:alignment}
            Let $S$ be a diagonal matrix of strictly positive and bounded entries, we have that $\forall \ G \in \mathbb{R}^{m \times n}_{\setminus 0}$:
        
            \begin{equation}
                \langle G, \operatorname{PolarFactor}(GS) \rangle > 0.
            \end{equation}
        
            \noindent Meaning that the projected update recovered by \turbomuon is always ``aligned'' with the raw gradient $G$. Thus, it yields a strict descent direction.
        \end{lemma}
        
        

        The proof is detailed in \cref{ap:proof_convergence}.

\paragraph{Geometric properties of the \turbomuon rescaling:}
We analyze the effect of the rescaling matrix $S$ used in \turbomuon pre-conditioning. This rescaling can be interpreted as inducing a norm that adapts to the structural coherence of the gradients, thereby normalizing the worst-case spectral contribution of feature clusters.

        The rescaling matrix $S$ is derived from the Gram matrix of the gradient, $H = G^\top G$. Geometrically, the entry $H_{jk} = \langle g_j, g_k \rangle$ quantifies the correlation between the gradient updates of feature column $j$ and feature column $k$. High values indicate that these features are being pushed in similar directions by the loss function.
        To extract a scalar measure of feature stability from this matrix without expensive eigendecompositions, we use the \textit{Gershgorin Circle Theorem} \citep{gershgorin1931uber}. The theorem states that the eigenvalues of $H$ are strictly bounded by its row sums. We define the Gershgorin coherence bound $\Lambda_j$ for the $j$-th feature as:
        
        \begin{equation}
            \Lambda_j \triangleq \sum_{k=1}^n |(G^\top G)_{jk}| = \underbrace{|(G^\top G)_{jj}|}_{\text{Feature Energy}} + \underbrace{\sum_{k \ne j} |(G^\top G)_{jk}|}_{\text{Structural Redundancy}}
        \end{equation}
        
        This quantity $\Lambda_j$ serves as an upper bound on the strength and redundancy of feature $j$ as it captures the magnitude of the feature's own gradient, and also its constructive interference with all other correlated features in the layer.
        By defining the scaling as $S = \operatorname{diag}(\Lambda)^{-1/2}$, the update of \turbomuon{}
        corresponds to steepest descent in an induced operator norm $\| \cdot \|_{\ell_2 \rightarrow S}$ (\cref{ap:proof_convergence}), that is explicitly shaped by these coherence bounds. We distinguish the following cases during optimization:
        
        \begin{itemize}
            \item \textbf{Penalization of Coherence:} When a group of features collapses (i.e., learns highly correlated representations), the off-diagonal terms in $G^\top G$ grow large. Consequently, the Gershgorin bound $\Lambda_j$ increases, and the effective step size $s_j$ for those features decreases. 
            \item \textbf{Amplification of Orthogonality:} Conversely, for features that are orthogonal to others, the redundancy term vanishes and the scaling is determined only by the feature's own energy, allowing for relatively larger updates.
        \end{itemize}
        
        Therefore, the steepest descent of \turbomuon ensures that the worst-case curvature estimate of every feature channel—bounded by the Gershgorin disk—is at most 1, resulting in a stabilized spectral update that 
        reduces sensitivity
        to feature collapse.

\paragraph{Limitations:} The previous analysis, as in~\cite{amsel2025polar,bernstein2024oldoptimizernewnorm}, relies on the assumption that the update direction exactly matches the $\operatorname{PolarFactor}$. In practice, however, the Newton-Schulz procedure only provides an approximation of this operator. Although we have shown that our method yields a tighter approximation, a gap remains between the theoretical analysis and the actual update used in practice~\cite{gonon2026insights}. 
Additionally, as in Muon, the effective step size $\eta$ depends on the gradient matrix $G$ and is implicitly absorbed into the optimizer’s learning rate. A dedicated analysis of how AOL pre-conditioning influences $\eta$ could provide useful insights for principled learning rate selection.

It would therefore be valuable to extend the analysis to explicitly account for approximate updates, for instance along the lines of approaches based on Linear Minimization Oracle (LMO)~\cite{pethick2025training}. We leave such an investigation to future work.
\section{Experiments: End-to-end training with \turbomuon}

While orthogonalization lies at the core of optimizers such as Muon, it is not the sole factor determining end-to-end performance. In this section, we first demonstrate that the proposed method can be used as a drop-in replacement of Muon, achieving equivalent performance while reducing training time across several benchmarks, including challenging and competitive speed-run settings. 
\begin{table}[t]
\centering
\begin{subtable}{0.47\linewidth}
    \centering
        \begin{tabular}{p{1mm} p{53pt}p{40pt}p{40pt}}
        \toprule
        \textbf{It} & \textcolor{turbomuoncolor}{\turbomuon} & \textcolor{dioncolor}{Muon+} & \textcolor{muoncolor}{Muon} \\
        \midrule
        
        
        
        

        2 
        & \cellcolor{red!15!white} 3.32$\pm 1e^{-3}$ 
        & \cellcolor{red!25!white} 3.35$\pm 3e^{-3}$ 
        & \cellcolor{red!20!white} 3.34$\pm 2e^{-3}$ \\
        
        3 
        & \cellcolor{red!10!blue!5} 3.29$\pm 1e^{-3}$ 
        & \cellcolor{red!12!white} 3.31$\pm 2e^{-3}$ 
        & \cellcolor{red!10!white} 3.30$\pm 2e^{-3}$ \\
        
        4 
        & \cellcolor{blue!12!white} 3.28$\pm 1e^{-3}$ 
        & \cellcolor{red!8!blue!8} 3.29$\pm 3e^{-3}$ 
        & \cellcolor{red!8!blue!8} 3.29$\pm 5e^{-3}$ \\
        
        5 
        & \cellcolor{blue!15!white} 3.28$\pm 1e^{-3}$ 
        & \cellcolor{blue!15!white} 3.28$\pm 3e^{-3}$ 
        & \cellcolor{blue!15!white} 3.28$\pm 1e^{-3}$ \\
        
        \bottomrule
        \end{tabular}
        \caption{\textbf{Our approach allows the removal of one Newton-Schulz iteration:} We trained a 144M GPT model on the FineWeb dataset~\cite{penedo2024fineweb} up to the performance of GPT-2 and reported validation loss. We reused the fastest existing training script and replaced the Newton-Schulz implementation without changing any other parameters. 
        }
        \label{tab:nanoGPT_accuracy}
\end{subtable}
\hfill
\begin{subtable}{0.51\linewidth}
    \centering
        \begin{tabular}{p{40pt}p{34pt}p{22pt}p{60pt}}
        \toprule
        \textbf{Dataset} & \textbf{Variant} & \textbf{Perf} & \textbf{Time}   \\
        \hline
        
        \multirow{2}{*}{\makecell[l]{speed-run\\Cifar10}} & \textcolor{dioncolor}{Muon+}   & \textbf{94.04}\%  & 1.359s \\
        & \textcolor{turbomuoncolor}{T-Muon}   & 94.03\%  & \textbf{1.348}s  (-0.8 \%) 
         \\
        \hline
        \multirow{2}{*}{\makecell[l]{speed-run\\nanoGPT}}& \textcolor{dioncolor}{Muon+}   & \textbf{3.28}  & 273.8s \\
         & \textcolor{turbomuoncolor}{T-Muon}   & \textbf{3.28}  & \textbf{266.0}s (-2.8\%) 
         \\
        \hline
        \multirow{2}{*}{\makecell[l]{ImageNet\\(1K)}} & \textcolor{dioncolor}{Muon+}    & \textbf{69.56}\%  & 422 m \\
         & \textcolor{turbomuoncolor}{T-Muon}   & 69.36\%  & \textbf{407} m (-3.6\%) 
         \\
          
        \bottomrule
        \end{tabular}
        \vspace{5mm}
    \caption{\textbf{End-to-end performance and runtime} for several benchmarks using \turbomuon with in-place replacement. \turbomuon (T-Muon) consistently achieves equivalent performance in less time (s: seconds, m: minutes).}
    \label{tab:end2endperf}
    \end{subtable}
\caption{\textbf{Improving Newton-Schulz convergence has an effective impact on end-to-end training.} With all other parameters fixed, the improved orthogonality translates to a lower final loss (\cref{tab:nanoGPT_accuracy}). Beyond 4 to 5 iterations, the loss marginally improves; hence, it is better to remove one iteration to reduce training time while maintaining results.}
\end{table}

We have seen in the previous section that  \turbomuon provides both a better approximation of the polar factor (\cref{fig:polar_error_NS})  and a large speedup (\cref{fig:runtime_NS}) for the core routine of Muon (i.e., the Newton-Schulz algorithm). However, this projection is only a part of the full training step, which also includes forward and backward propagation. In this section, we show that an in-place replacement of Muon by \turbomuon achieves the same performance in less training time on several benchmarks. We choose two "speed-run" settings in NLP (Modded-NanoGPT) and vision (CIFAR-10), and a standard ImageNet-1K setting.

We use the term “speed-run” to refer to highly optimized end-to-end training benchmarks whose objective is to reach a predefined target performance in the minimum possible wall-clock time. Unlike standard evaluations, these benchmarks are explicitly engineered to eliminate avoidable overheads in the model, data pipeline, and training loop, leaving little room for further acceleration. Consequently, even modest reductions in total runtime are meaningful, as they represent measurable gains on top of already highly optimized training recipes in which the Muon optimizer step has already been minimized. In this context, speed-run benchmarks provide a conservative test of whether improvements to Muon’s orthogonalization routine translate into practical end-to-end acceleration.

\paragraph{NanoGPT "speed-run":}
We evaluate \turbomuon on the Modded-NanoGPT benchmark~\cite{modded_nanogpt_2024}, a highly optimized “speed-run” setup for training a 124M GPT model on FineWeb until it reaches a validation cross-entropy of 3.28, replicating the estimated performance of GPT-2~\cite{radford2019language}. This benchmark is particularly challenging for our method, as it is tuned to minimize the overhead of Muon while emphasizing reproducible runtime measurements.
We reuse the most recent training recipe and replace Muon with \turbomuon, with a configurable number of Newton-Schulz iterations. All other hyperparameters are kept unchanged to assess drop-in compatibility. Experiments are run on 4×H100, reporting average loss and runtime over 10 runs.
An important point is that \turbomuon achieves equivalent performance with one fewer iteration (\cref{tab:nanoGPT_accuracy}) and reduces total runtime from $273.75 \pm 0.14$ s for the fastest version \textcolor{dioncolor}{Muon+} to $266.00 \pm 0.10$ s for \textcolor{turbomuoncolor}{\turbomuon{}}(\cref{tab:end2endperf}). Additional details and experiments are provided in \cref{ap:nanogpt-WR,ap:trainingregimes}.

\paragraph{CIFAR-10 "speed-run":}

We further evaluate \turbomuon on the CIFAR-10 speed-run benchmark~\cite{Jordan2024cifar10airbench}, which targets 94\% validation accuracy in the minimum possible time on a single H100 GPU. This setting involves training a convolutional neural network (CNN), where Muon 
reshapes convolutional gradient updates into 2D matrices before applying iterative orthogonalization.
This provides a complementary regime to evaluate AOL preconditioning.

As shown in \cref{tab:end2endperf}, AOL pre-conditioning achieves the same performance while reducing runtime from $1.359s$ to $1.348s$ (0.8\%). Despite the highly optimized nature of this benchmark, this gain highlights the drop-in efficiency of \turbomuon and its applicability across architectures.

\paragraph{ImageNet-1K:} To evaluate the influence of \turbomuon on a larger dataset,
we also evaluate a drop-in replacement on an ImageNet-1K baseline using a ViT-Tiny/16 model from \texttt{timm}~\cite{rw2019timm}. 
\cref{tab:end2endperf} shows that a 3.6\% runtime improvement is achieved with a similar final performance. 

\section{Conclusion and Perspectives}

In this work, we propose AOL pre-conditioning 
to improve 
the efficiency of orthogonality-based optimizers such as Muon.
Under practical approximation budgets (i.e., $t \leq 5$), 
the proposed approach yields both faster and more accurate orthogonalization: Empirically, AOL significantly reduces the initial polar error compared to Frobenius normalization on large random matrices (\cref{fig:precond_bench}), and enables the removal of one Newton-Schulz iteration while still achieving equal or lower error than the five-step Muon and Muon+ baselines (\cref{fig:polar_error_NS}). 
These gains translate into consistent runtime improvements, with AOL pre-conditioning acting as a drop-in replacement that requires no hyperparameter tuning:  Across speed-run and standard benchmarks, \turbomuon preserves performance while reducing computation time. In particular, 
we achieve near-identical performance on NanoGPT and CIFAR-10 speed-runs, 
with up to 3\% speedup,
and similar trends on larger-scale settings such as ImageNet-1K.
At medium scale (\cref{tab:1b_runtime}), Turbo-Muon can even provide step-time speedups of 8–10\%.

Beyond empirical performance, we provide a geometric interpretation of AOL pre-conditioning through a steepest descent characterization. We show that \turbomuon updates are close to the solution of a steepest descent problem under a structured operator norm. This perspective reveals that AOL pre-conditioning implicitly controls the spectral geometry of the update. As a consequence, \turbomuon reduces sensitivity to feature collapse and promotes more stable and balanced optimization dynamics.

\paragraph{Broader impact.}
By reducing the computational cost of orthogonality-constrained optimization without sacrificing performance or requiring additional tuning, AOL pre-conditioning contributes to more efficient training of large-scale models. Moreover, AOL pre-conditioning is orthogonal to many existing Muon improvements; in particular, we show in the appendix (\cref{ap:tridao}) that it can be naturally combined with the very recent Gram-based Newton-Schulz computation method~\cite{GramNewtonSchulz}, making it a complementary building block rather than a competing approach.


\ifarxiv

\section*{Acknowledgments}

This work was carried out within the DEEL project,\footnote{\url{https://www.deel.ai/}} which is part of IRT Saint Exupéry and the ANITI AI cluster. The authors acknowledge the financial support from DEEL's Industrial and Academic Members and the France 2030 program - Grant agreements n°ANR-10-AIRT-01 and n°ANR-23-IACL-0002.

\noindent This work was granted access to the HPC resources of IDRIS under the allocation 2025-AD011016381 made by GENCI. 
\fi

{\small
\bibliographystyle{ieeenat_fullname}
\bibliography{9_references}
}

\onecolumn
\appendix
\clearpage \appendix

\section{On the convergence of \turbomuon}
\subsection{Link with norm  constrained steepest descent}
\label{ap:proof_convergence}
    In this section, we build upon the work of \citet{bernstein2024oldoptimizernewnorm}, as they show that Muon's update (without accumulation) solves a steepest descent direction in a $\ell_2 \rightarrow \ell_2$ induced operator norm.
    As a starting point, we define how the norm induced by our preconditioning differs from the original one. We show in \cref{supp:descent_proof}
    that preconditioning does not bias the overall optimization process.
    
    \hfill
    
    Following \cite{bernstein2024oldoptimizernewnorm}, we formulate the steepest descent update using a penalized objective, which is equivalent to the constrained formulation for a specific step size $\eta$, i.e.:
    
    \begin{equation}
        \arg \min_{\Delta W \in \mathbb{R}^{m \times n}} \left[ \langle G, \Delta W \rangle + \frac{\lambda}{2} \| \Delta W\|_{\alpha \rightarrow \beta}^2 \right],
    \label{eq:steepest_descent}
    \end{equation}
    
    \noindent where $\langle \cdot , \cdot \rangle$ denotes the Frobenius inner product. For any matrix $M \in \mathbb{R}^{m \times n}$, and normed spaces $(\mathbb{R}^{d_\mathrm{in}}, \| \cdot \|_\alpha)$ and $(\mathbb{R}^{d_\mathrm{out}}, \| \cdot \|_\beta)$, the ``$\alpha$ to $\beta$'', we recall the induced operator norm as:
    
    \begin{equation}
        \| M \|_{\alpha \rightarrow \beta} = \max_{x \in \mathbb{R}^{d_\mathrm{in}}} \frac{\| M x \|_\beta}{\| x \|_\alpha}.
    \end{equation}
    
    \noindent Importantly, as explicited in Proposition 5 of \cite{bernstein2024oldoptimizernewnorm}. When $\alpha = 2$ and $\beta = 2$, we have that the solution to Equation~\ref{eq:steepest_descent} is:

    \begin{equation*}
        \Delta W_{\| \cdot \|_{\ell_2 \rightarrow \ell_2}} = -\eta \operatorname{PolarFactor}(G),
    \end{equation*}
    
    \noindent with $\eta$ the sum of the singular values of $G$ divided by $\lambda$. In this paper, we introduce the \turbomuon update as $\operatorname{PolarFactor}(G S)$. With $S$ the diagonal matrix composed of AOL preconditioning coefficients (previously denoted as $s$), which are strictly positive and bounded in practice. Here we show that this update is the solution of a slightly different steepest descent direction problem formulated using a rescaled norm.
    
        
    
    

    
    \begin{proposition}[\turbomuon in the Steepest Descent framework]
    For any gradient matrix $G \in \mathbb{R}^{m \times n}_{\setminus 0}$, and any sharpness $\lambda > 0$, consider the problem:
        
        \begin{equation}
        \begin{aligned}
            & \arg \min_{\Delta W \in \mathbb{R}^{m \times n}} \left[  \langle G, \Delta W \rangle_F + \frac{\lambda}{2} \| \Delta W \|_{\ell_2 \rightarrow S}^2 \right] \\
             \text{with} & \ \| M \|_{\ell_2 \rightarrow S} := \sup_{x \in \mathbb{R}^{n}_{\setminus 0}} \frac{\| M S^{-1} x \|_2}{\| x\|_2}.
        \end{aligned}
        \end{equation}
        \noindent This problem is solved with a step size $\eta = \frac{1}{\lambda} \operatorname{tr}(\Sigma_{GS})$, with $GS=U \Sigma_{GS} V^\top$ the SVD decomposition of $GS$, with $U$ and $V$ unitary,  and an update:
    
        \begin{equation}
            \Delta W_{\| \cdot \|_{\ell_2 \rightarrow S}} = -\eta \operatorname{PolarFactor}(G S) S,
        \end{equation}
    
        \noindent This solution is unique if and only if $G$ is of full rank. 

    \end{proposition}
    
    This proposition with $S$, the diagonal matrix filled with rescaling values obtained from \cref{eq:aol} (strictly positive),  generalizes a scaled variant of \turbomuon as an optimizer that yields steepest descent directions in an induced norm that is dependent on AOL rescalings, as per the framework of \cite{bernstein2024oldoptimizernewnorm}.
    
    \begin{proof}
    Given the problem 
    
        \begin{equation}
        \begin{aligned}
            & \arg \min_{\Delta W \in \mathbb{R}^{m \times n}} \left[  \langle G, \Delta W \rangle_F + \frac{\lambda}{2} \| \Delta W \|_{\ell_2 \rightarrow S}^2 \right] \\
             \text{with} & \ \| M \|_{\ell_2 \rightarrow S} := \sup_{x \in \mathbb{R}^{n}_{\setminus 0}} \frac{\| M S^{-1} x \|_2}{\| x\|_2}.
        \end{aligned}
        \end{equation}

    By setting $\Delta W' = \Delta W S^{-1}$, we have 
$\| \Delta W \|_{\ell_2 \rightarrow S}^2 = \| \Delta W' \|_{\ell_2 \rightarrow \ell_2}^2$.

Besides using the cyclic property of the trace along with the symmetry of $S$, the inner Frobenius product on the left-hand side of this expression can be modified by:
    \begin{equation}
    \label{eq:frobprod}
    \langle G, \Delta W  \rangle_F = \langle G, \Delta W' S  \rangle_F = \langle G S, \Delta W' \rangle_F
    \end{equation}

    The problem is thus equivalent to minimizing:
    
    \begin{equation}
        \arg \min_{\Delta W' \in \mathbb{R}^{m \times n}} \left[ \langle GS, \Delta W' \rangle + \frac{\lambda}{2} \| \Delta W'\|_{2 \rightarrow 2}^2 \right] 
    \end{equation}
    
    From Proposition 5 of \citet{bernstein2024oldoptimizernewnorm}, the solution of this problem is given by:

    \begin{equation}
       \Delta W' = - \eta \operatorname{PolarFactor}(GS)
    \end{equation}
    with a step size $\eta = \frac{1}{\lambda} \operatorname{tr}(\Sigma_{GS})$.
    \begin{equation}
       \Delta W = \Delta W' S= - \eta \operatorname{PolarFactor}(GS) S
    \end{equation}

    The rescaled \turbomuon update is thus close to the solution of a steepest descent framework under the  $\| . \|_{\ell_2 \rightarrow S}$ norm.
    
    \end{proof}

    \paragraph{Discussion:} One main difference between this proof and the update used in \turbomuon is the presence of an extra scaling factor $S$. Applying this scaling vector after the polar factorization would yield an optimizer with dynamics different from Muon's. The exploration of such a novel optimizer is left for future work. 
    
    \subsection{AOL preconditioning cannot lead to Muon divergence}
    \label{supp:descent_proof}
    
        We start by recalling \cref{lem:alignment}:

        \begin{lemma*}[Lemma 1]
        \label{lem:alignment2}
            Let $S$ be a diagonal matrix of strictly positive and bounded entries, we have that $\forall \ G \in \mathbb{R}^{m \times n}_{\setminus 0}$:
        
            \begin{equation}
                \langle G, \operatorname{PolarFactor}(GS) \rangle > 0.
            \end{equation}
        
            \noindent Meaning that the projected update recovered by \turbomuon is always ``aligned'' with the raw gradient $G$. Thus, it yields a strict descent direction.
        \end{lemma*}
        
        \begin{proof}
        Write $X:=G S$. By definition of the (left) polar factor,
        \[
            \operatorname{PolarFactor}(X)=X\,(X^\top X)^{-1/2},
        \]
        where $(\cdot)^{-1/2}$ denotes the unique positive–semidefinite inverse square root on the support of $X^\top X$.
        Hence
        \begin{equation*}
        \begin{aligned}
            \langle G,\operatorname{PolarFactor}(G S)\rangle
            &= \operatorname{tr}\!\big(G^\top\,\operatorname{PolarFactor}(X)\big) \\
            &= \operatorname{tr}\!\big(S^{-1}\,X^\top \operatorname{PolarFactor}(X)\big),
        \end{aligned}
        \end{equation*}
        since $G^\top = S^{-1}X^\top$ and $S^{-1}\succ0$ (as per \cref{eq:aol}). Using $X^\top \operatorname{PolarFactor}(X)=X^\top X\,(X^\top X)^{-1/2}=(X^\top X)^{1/2}$, we get
        \[
            \langle G,\operatorname{PolarFactor}(GS)\rangle
            = \operatorname{tr}\!\big(S^{-1}\,(X^\top X)^{1/2}\big).
        \]
        
        The matrix $(X^\top X)^{1/2}\succeq0$ and is nonzero because $G\neq0$ and $S$ is invertible, while $S^{-1}\succ0$. For positive semidefinite $S\neq0$ and positive definite $T$, $\operatorname{tr}(TS)=\operatorname{tr}(T^{1/2}ST^{1/2})>0$. Therefore $\operatorname{tr}\!\big(S^{-1}(X^\top X)^{1/2}\big)>0$ and the claim follows.
        The argument does not require $G$ to be square or full rank. If $X$ is rank deficient, interpret $(X^\top X)^{-1/2}$ as the Moore–Penrose inverse square root, so the identity $X^\top \operatorname{PolarFactor}(X)=(X^\top X)^{1/2}$ holds on the support of $X^\top X$, and the same positivity conclusion applies.
        \end{proof}

        \begin{figure*}
            \centering
            \begin{subfigure}{0.32\textwidth}
                \includegraphics[width=\textwidth]{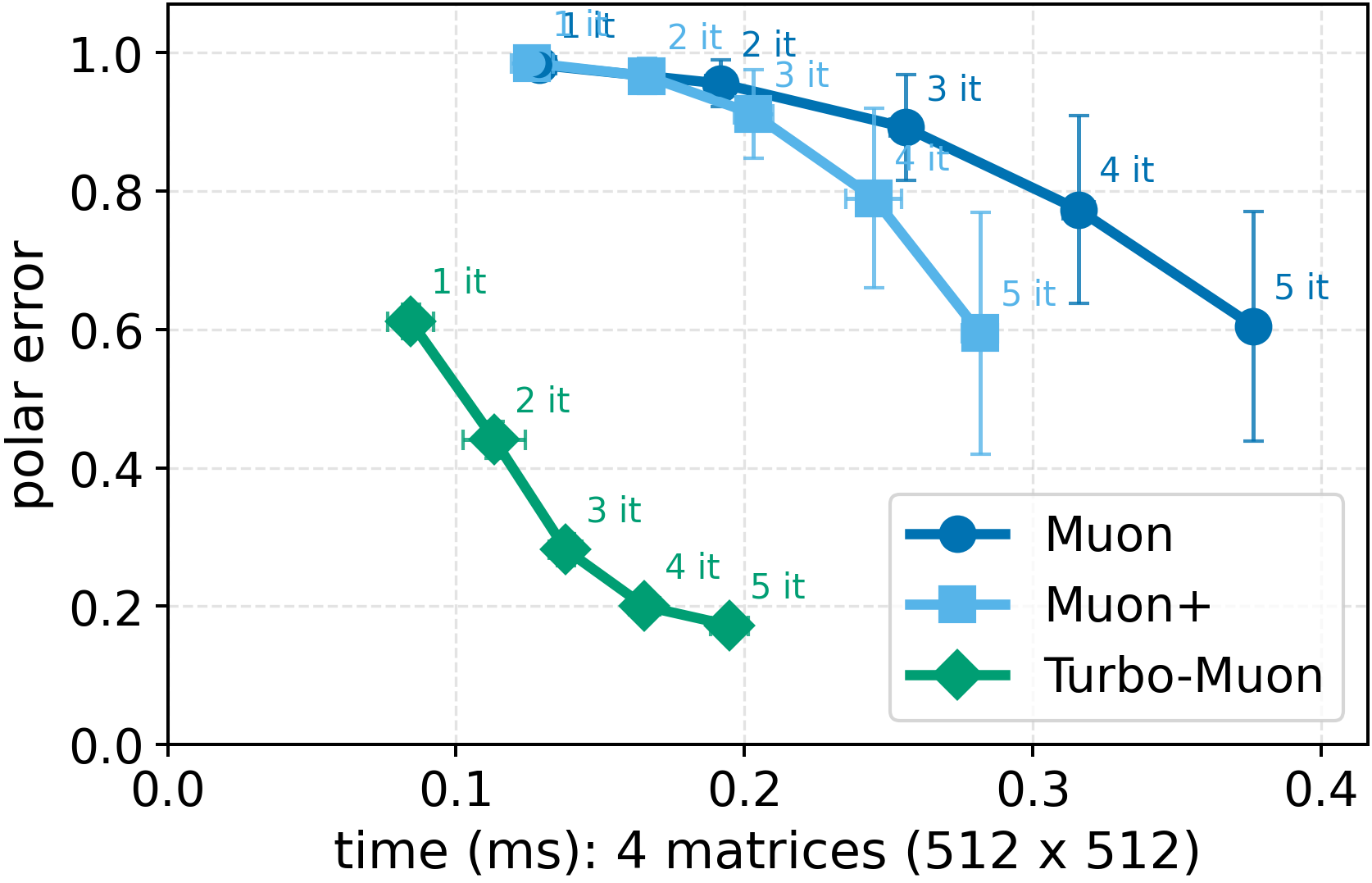}
                \caption{Levy distribution: $\alpha = 1.0$, $\beta = 0$}
            \end{subfigure}
            \hfill
            \begin{subfigure}{0.32\textwidth}
                \includegraphics[width=\textwidth]{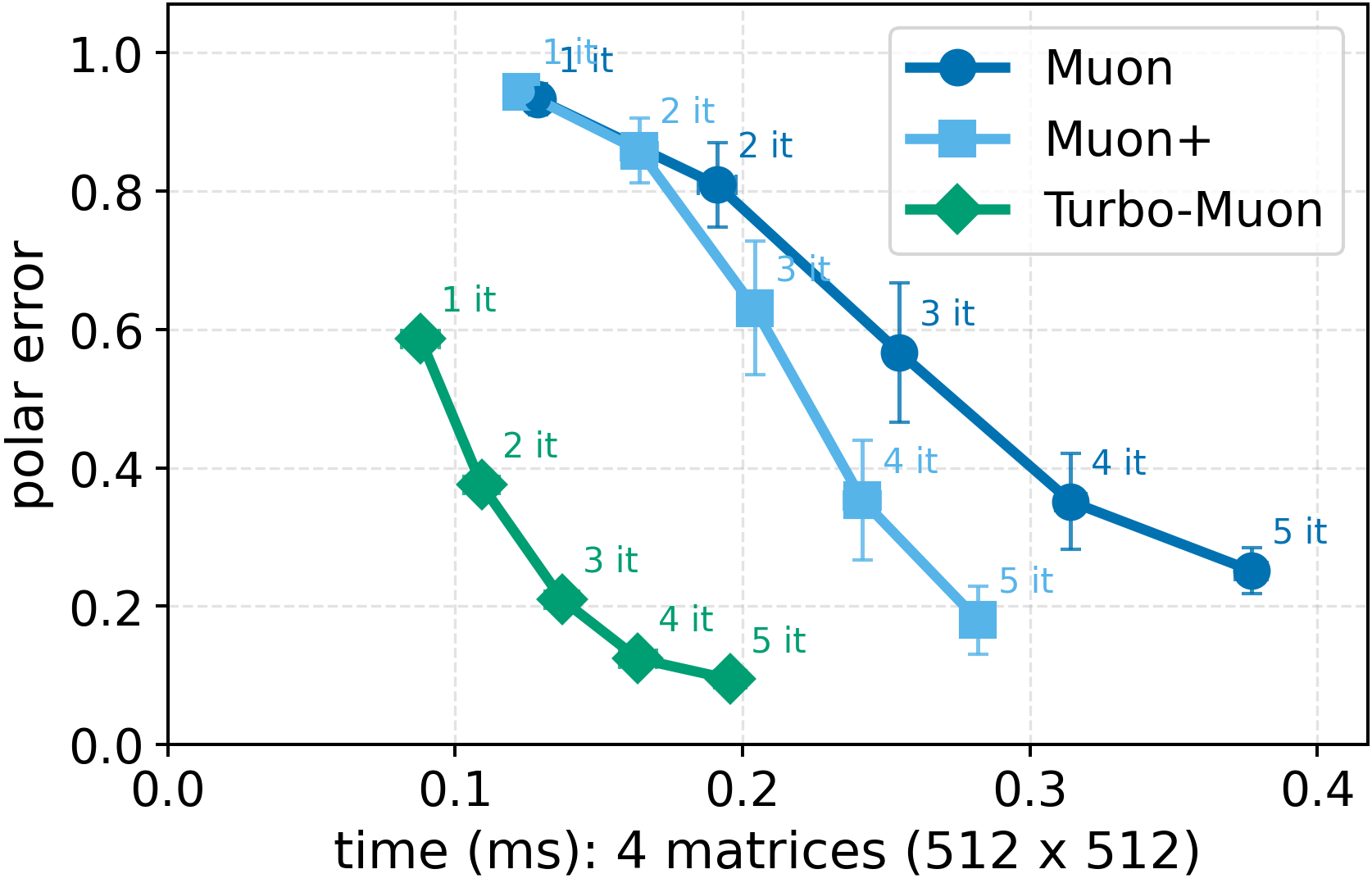}
                \caption{Levy distribution: $\alpha = 1.5$, $\beta = 0$}
            \end{subfigure}
            \hfill
            \begin{subfigure}{0.32\textwidth}
                \includegraphics[width=\textwidth]{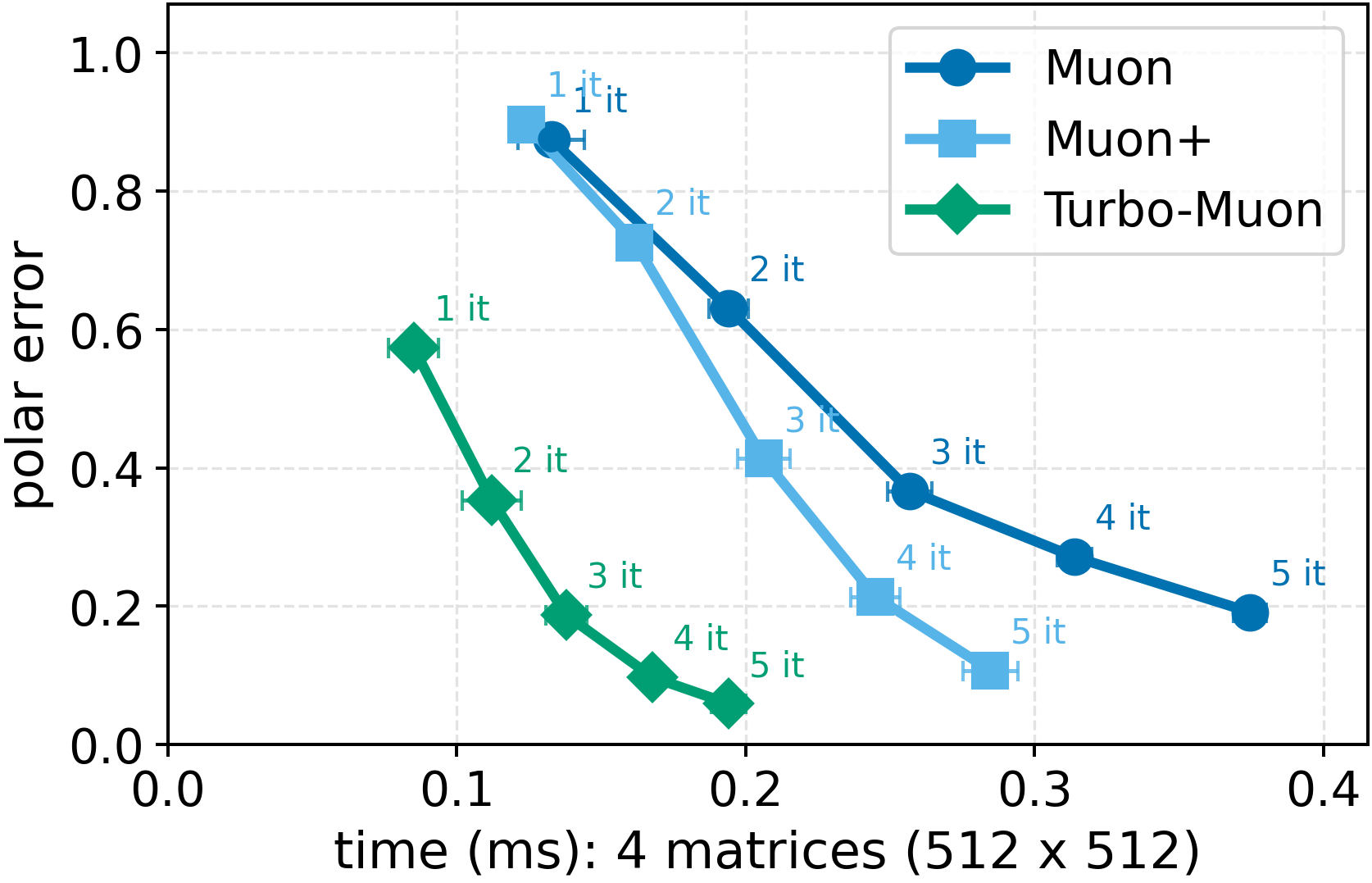}
                \caption{Levy distribution: $\alpha = 2.0$, $\beta = 0$}
            \end{subfigure}
            \caption{We reproduced the \cref{fig:pareto8192} using heavy-tailed distributions. Despite an overall degradation in the resulting polar error, \turbomuon outperforms existing approaches by a significant margin.}
            \label{fig:paretos_levy}
        \end{figure*}

\section{What About Gradient Distributions ?}
    \label{ap:real_gradients}
    
    \subsection{Heavy-tailed matrices}
    In this section, we run the complementary experiments from \cref{ssec:AOLxNS} and \cref{ap:bias_analysis} on matrices sampled from a Levy distribution of parameters $\alpha$ and $\beta$. This choice is motivated by observations made in \citet{Simsekli2019TailIndex}, which found that the tail index of gradients belongs to $\alpha \in [1.0, 1.8]$ in typical training of classification neural networks. Independently, \citet{kunstner2024heavy} have also shown similar observations on transformers trained on language modeling tasks.

    \hfill
      
    We reproduce the protocol of \cref{fig:pareto8192} from the main paper and report results using three different tail-indices $\{ 1.0, 1.5, 2.0 \}$ in \cref{fig:paretos_levy}, ensuring we cover realistic values. We set $\beta=0$ since, despite the fact that the gradient expectation over timesteps is non-null, we made the assumption that its expectation over neurons is null: high values for $\beta$ induce a rank collapse toward rank 0 (i.e., a constant), which is unrealistically pathological. Intuitively, when $\alpha=2.0$, the distribution is nearly Gaussian, and for $\alpha\leq 1.0$, estimation of the mean becomes an inconsistent estimator~\cite{wiki_Levy_distribution}. For all tested values of $\alpha$, \turbomuon outperforms Muon/Muon+ by a significant margin in realistic settings ($t \leq 5$).

    \begin{figure*}
        \centering
        \begin{subfigure}{0.32\textwidth}
            \includegraphics[width=\textwidth]{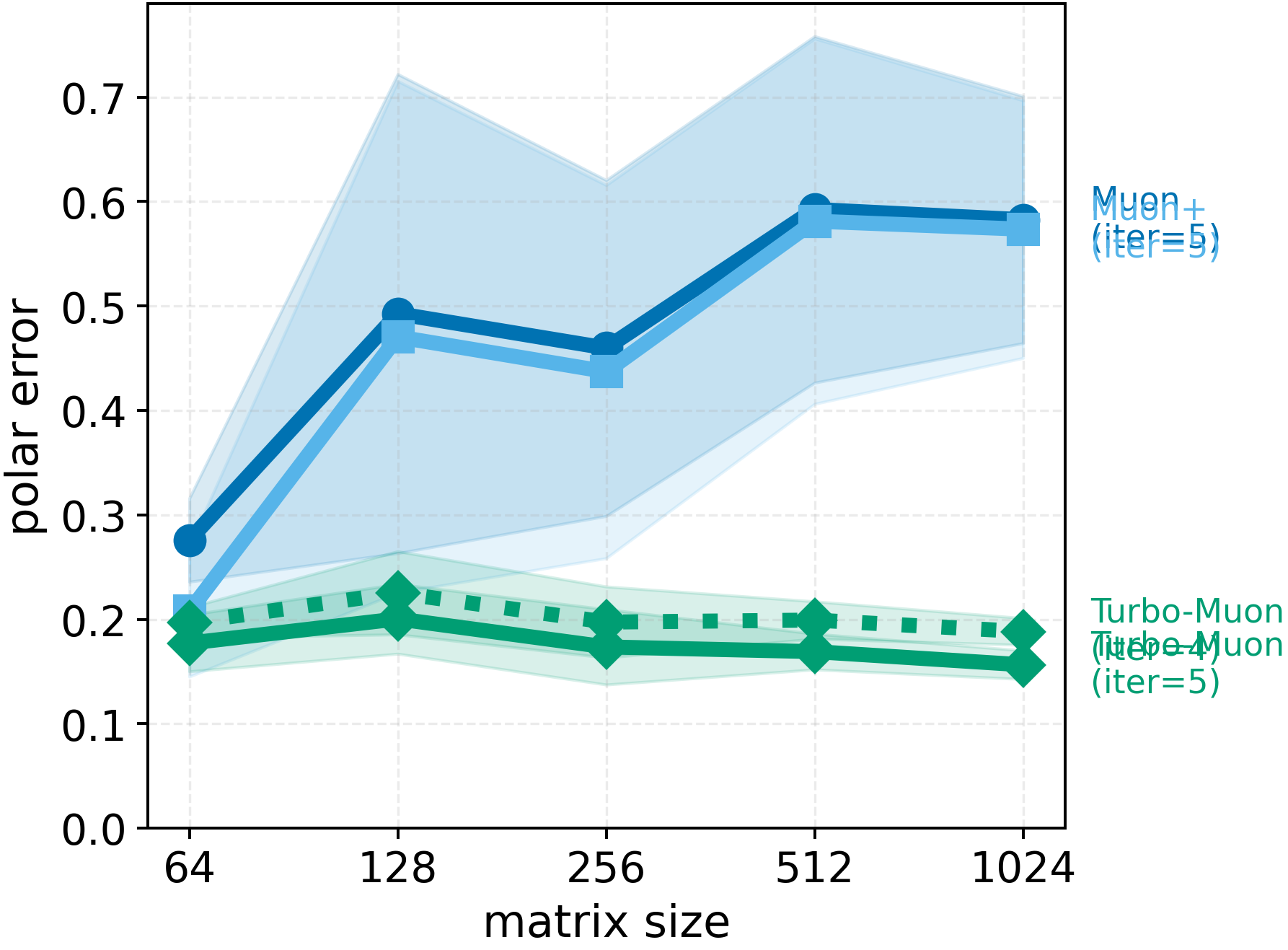}
            \caption{Levy distribution: $\alpha = 1.0$, $\beta = 0$}
        \end{subfigure}
        \hfill
        \begin{subfigure}{0.32\textwidth}
            \includegraphics[width=\textwidth]{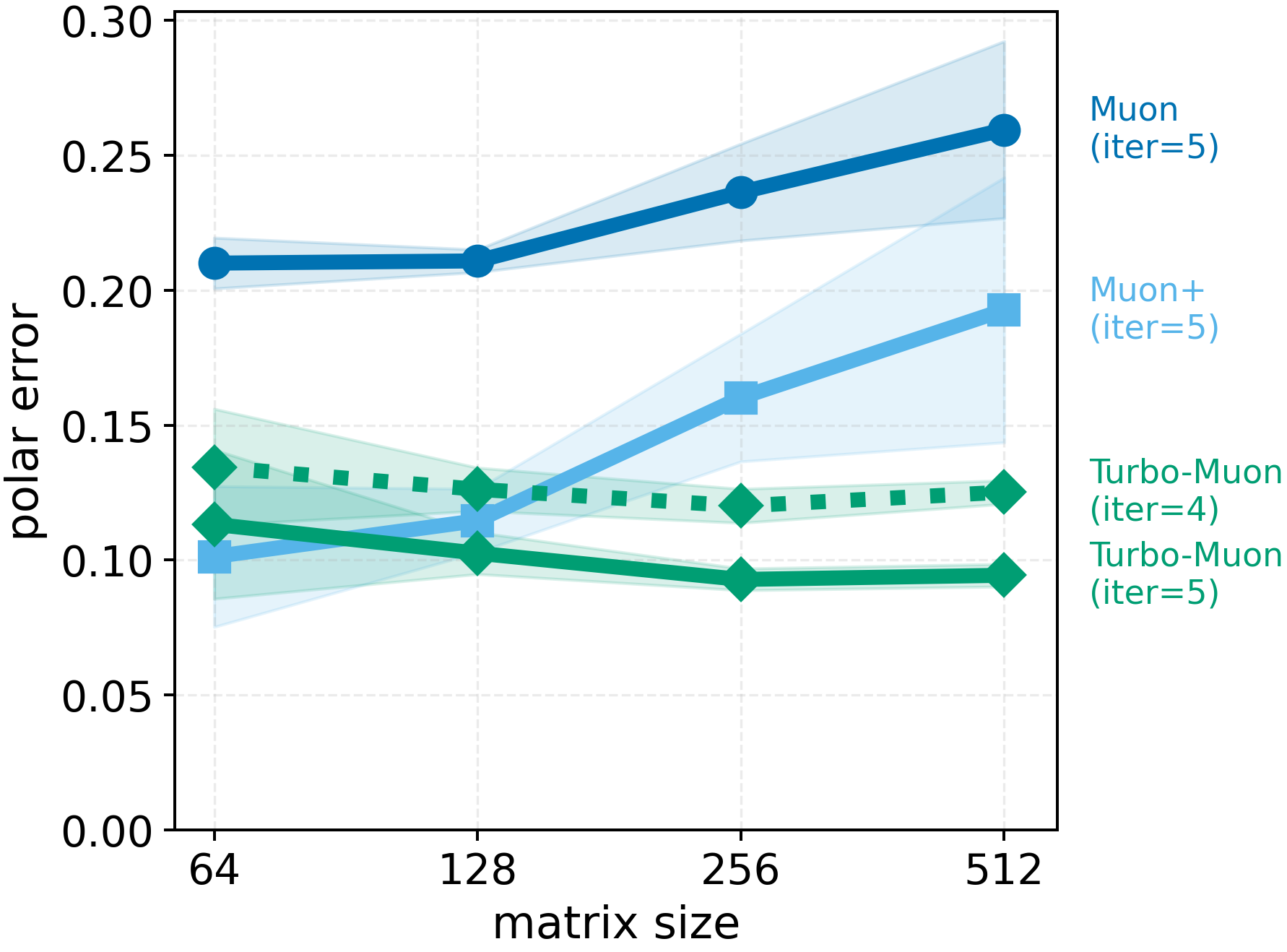}
            \caption{Levy distribution: $\alpha = 1.5$, $\beta = 0$}
        \end{subfigure}
        \hfill
        \begin{subfigure}{0.32\textwidth}
            \includegraphics[width=\textwidth]{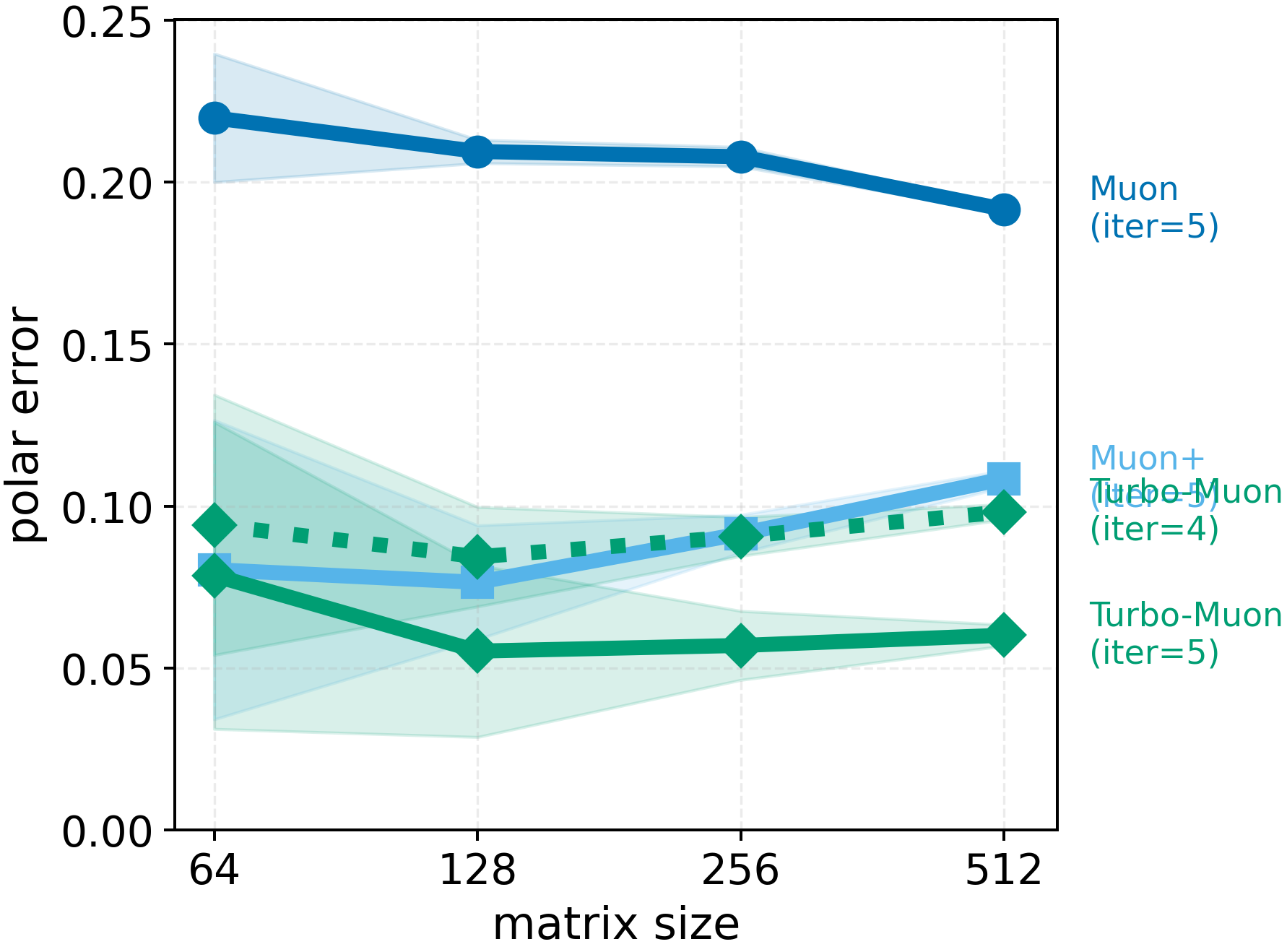}
            \caption{Levy distribution: $\alpha = 2.0$, $\beta = 0$}
        \end{subfigure}
        \caption{We reproduced the \cref{fig:polar_error_NS} using heavy-tailed distributions. Despite an overall degradation in the resulting polar error, \turbomuon outperforms existing approaches by a significant margin.}
        \label{fig:polar_bench_levy}
    \end{figure*}

    Similarly, we reproduced the result of \cref{fig:polar_error_NS} from the main paper for Levy distributions and reported results in \cref{fig:polar_bench_levy}. Again, we observe a global degradation of performances, with a lower degradation for \turbomuon. Notably, in these experiments, we observed that the PyTorch implementation of the singular value decomposition was unstable in these regimes, especially for large matrices. Thus limiting these experiments to matrices of size $512\times 512$.
    
    \begin{figure*}
        \centering
        \centering
        \begin{subfigure}{0.32\textwidth}
            \includegraphics[width=\textwidth]{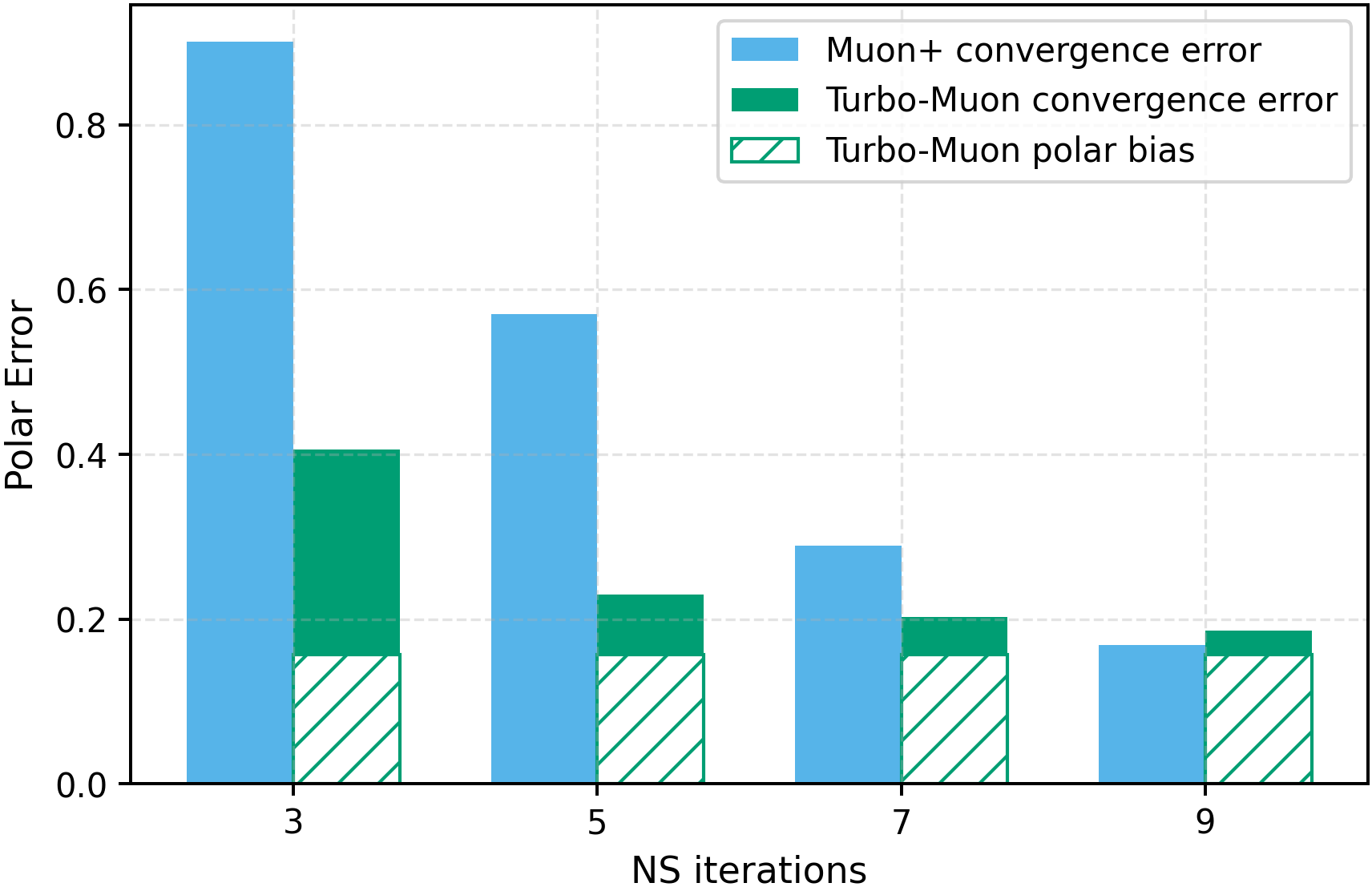}
            \caption{Levy distribution: $\alpha = 1.0$, $\beta = 0$}
        \end{subfigure}
        \hfill
        \begin{subfigure}{0.32\textwidth}
            \includegraphics[width=\textwidth]{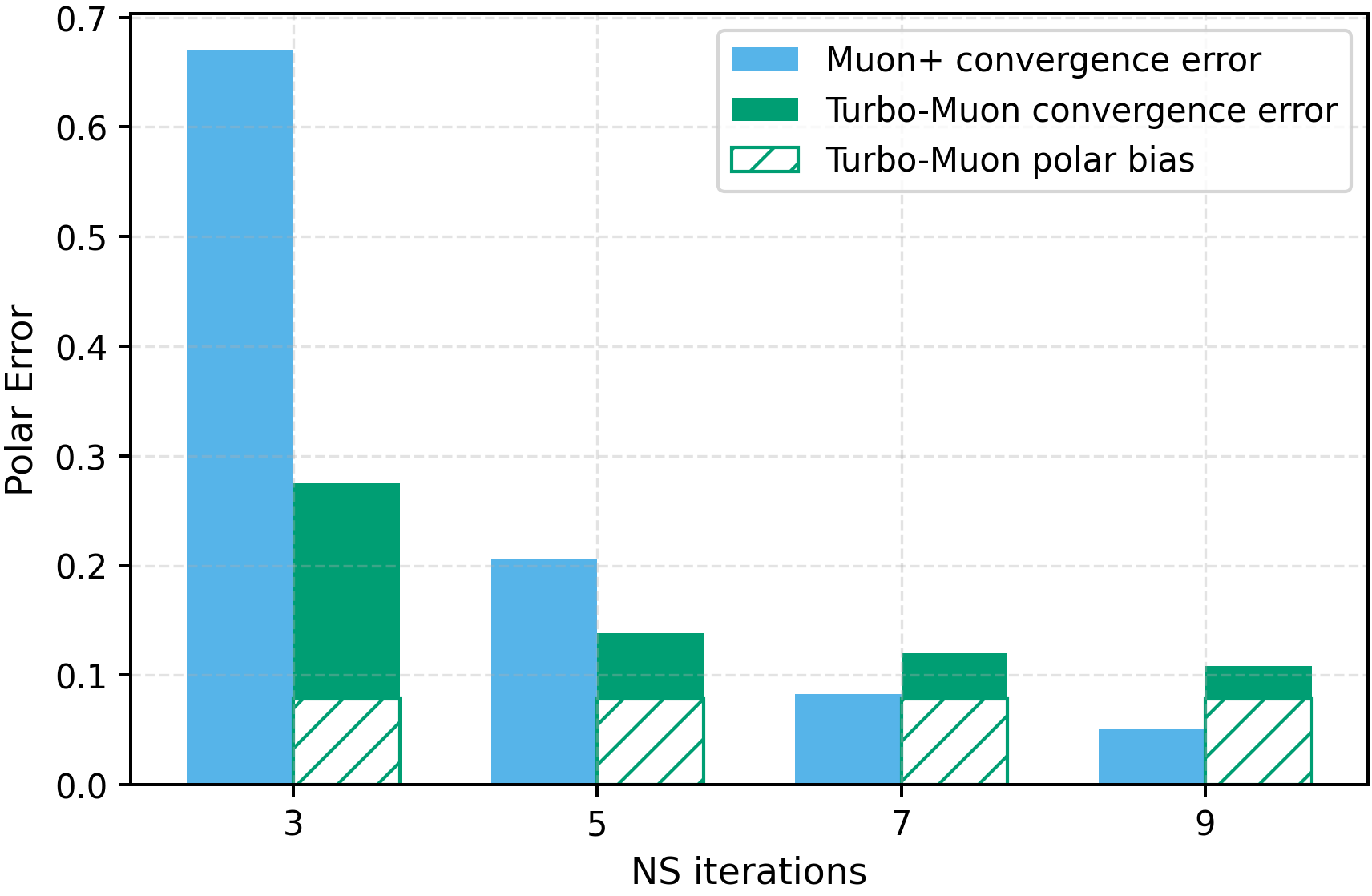}
            \caption{Levy distribution: $\alpha = 1.5$, $\beta = 0$}
        \end{subfigure}
        \hfill
        \begin{subfigure}{0.32\textwidth}
            \includegraphics[width=\textwidth]{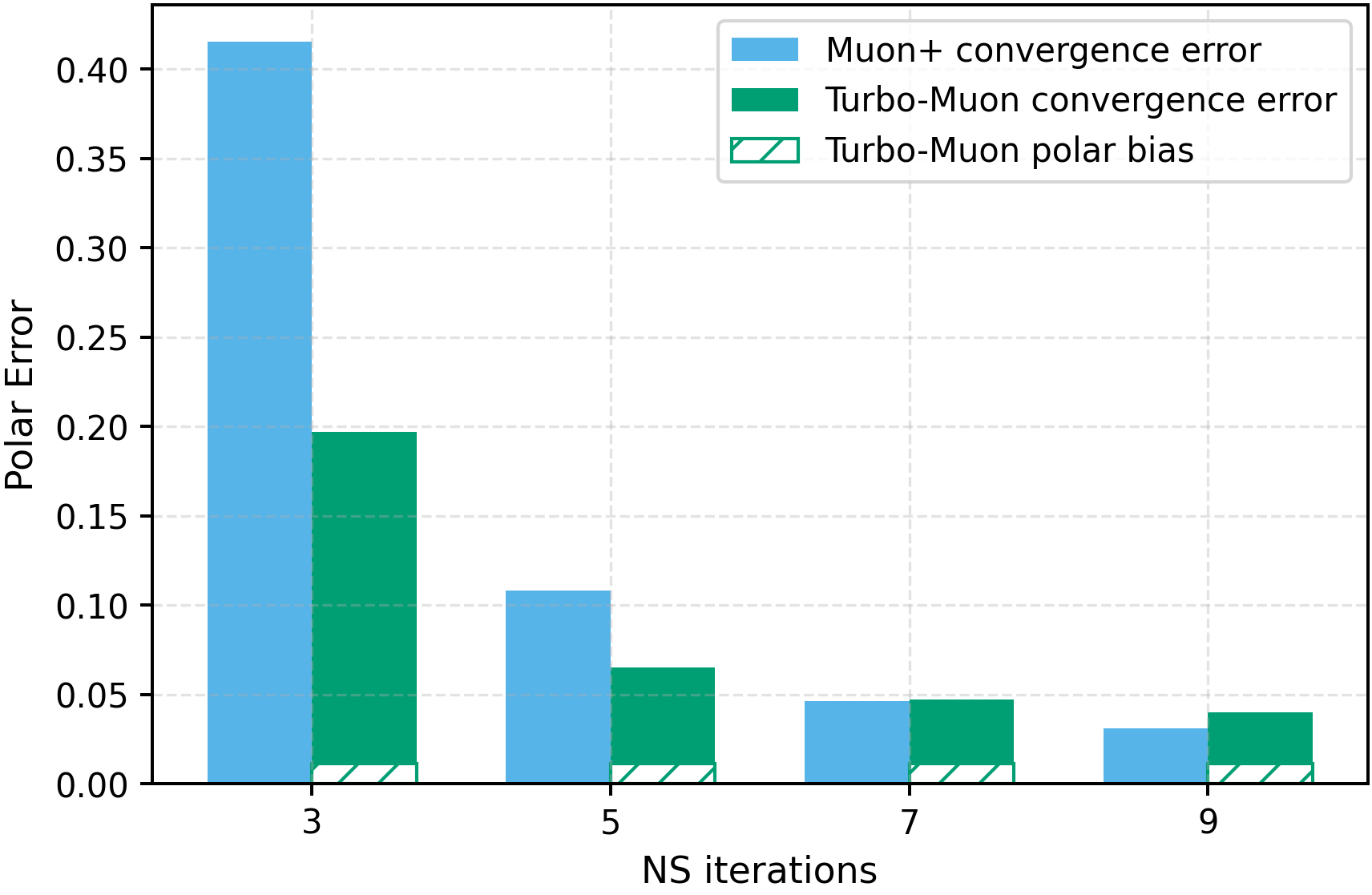}
            \caption{Levy distribution: $\alpha = 2.0$, $\beta = 0$}
        \end{subfigure}
        \caption{Reproduction of \cref{fig:bias_converge} with heavy tailed distributions. This extreme regime--where the singular value decomposition is at its numerical limit-- shows that the tail index correlates with the bias error induced by our preconditioning.}
        \label{fig:bias_converge_levy}
    \end{figure*}
    
    Interestingly, in these regimes, the orthogonality error of \turbomuon is very low, highlighting that the remaining error is not due to a lack of convergence but to an increased bias for poorly conditioned matrices. This can be observed by reproducing the \cref{fig:bias_converge_levy} experiment from the paper: for instance, when $\alpha=1$, the polar bias can reach up to 0.15, which is significantly higher than the bias found with matrices sampled from normal distributions. Yet the drastically improved convergence still leads to a lower polar error compared to Muon+, which requires nine iterations (27 matrix multiplications) to reach a lower polar error than \turbomuon. The polar bias observed in \cref{tab:asymptotic} suggests that empirical gradients have a tail index $1.0 \leq \alpha \leq 1.5$, which is coherent with the observations of \citet{Simsekli2019TailIndex}.

\section{Result for various matrix sizes}
    \label{ap:more_normal_results}
    
    In this section, we show that the results displayed in \cref{fig:pareto8192} of the main paper hold for various matrix sizes sampled from the Normal distribution. Interestingly, two effects are observed: on one side, our preconditioning yields improvements in the polar error (shifting \turbomuon on the vertical axis), especially for large matrices; however, this effect becomes smaller for small matrices. On the other side, the triton kernel of 
    \cref{alg:muon} line 5
    improves runtime (shifting \turbomuon on the horizontal axis) on small matrices, as those lead to communication-bound kernels. Similarly, this effect becomes smaller for large matrices, as those lead to compute-bound kernels. This combination ensures a dominance of \turbomuon in all configurations.
    
    \begin{figure*}
        \centering
        \includegraphics[width=0.32\textwidth]{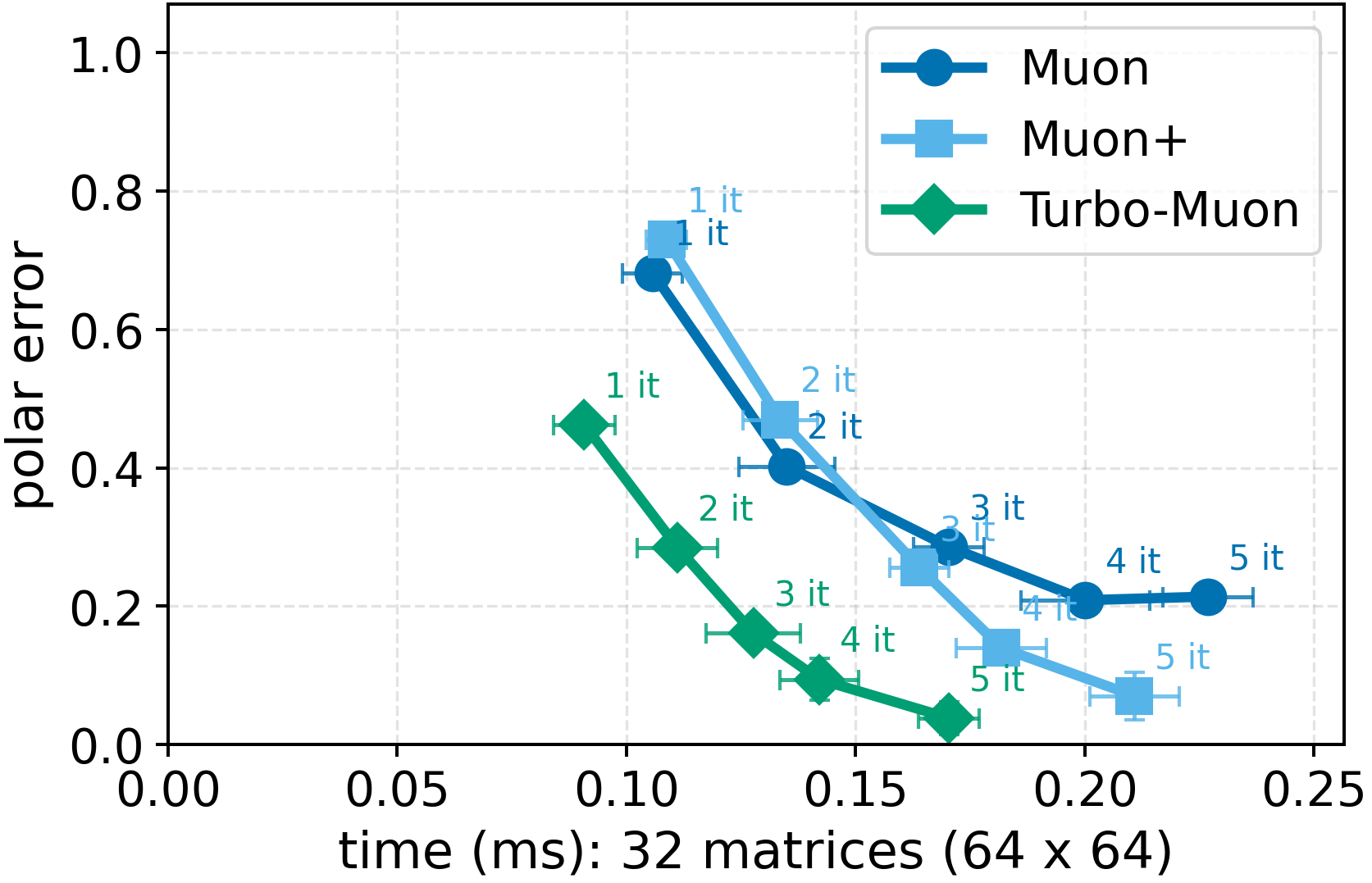}
        \hfill
        \includegraphics[width=0.32\textwidth]{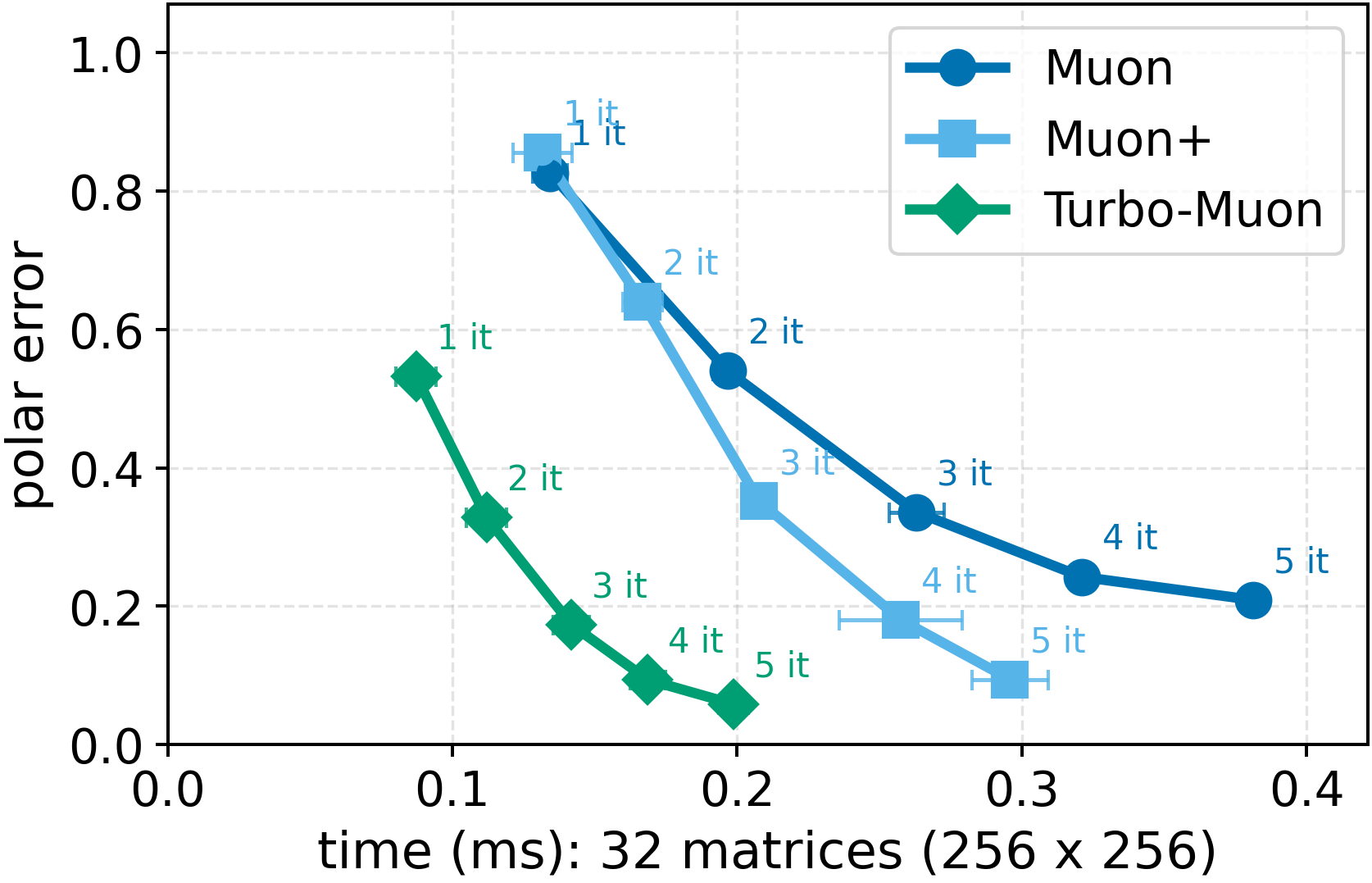}
        \hfill
        \includegraphics[width=0.32\textwidth]{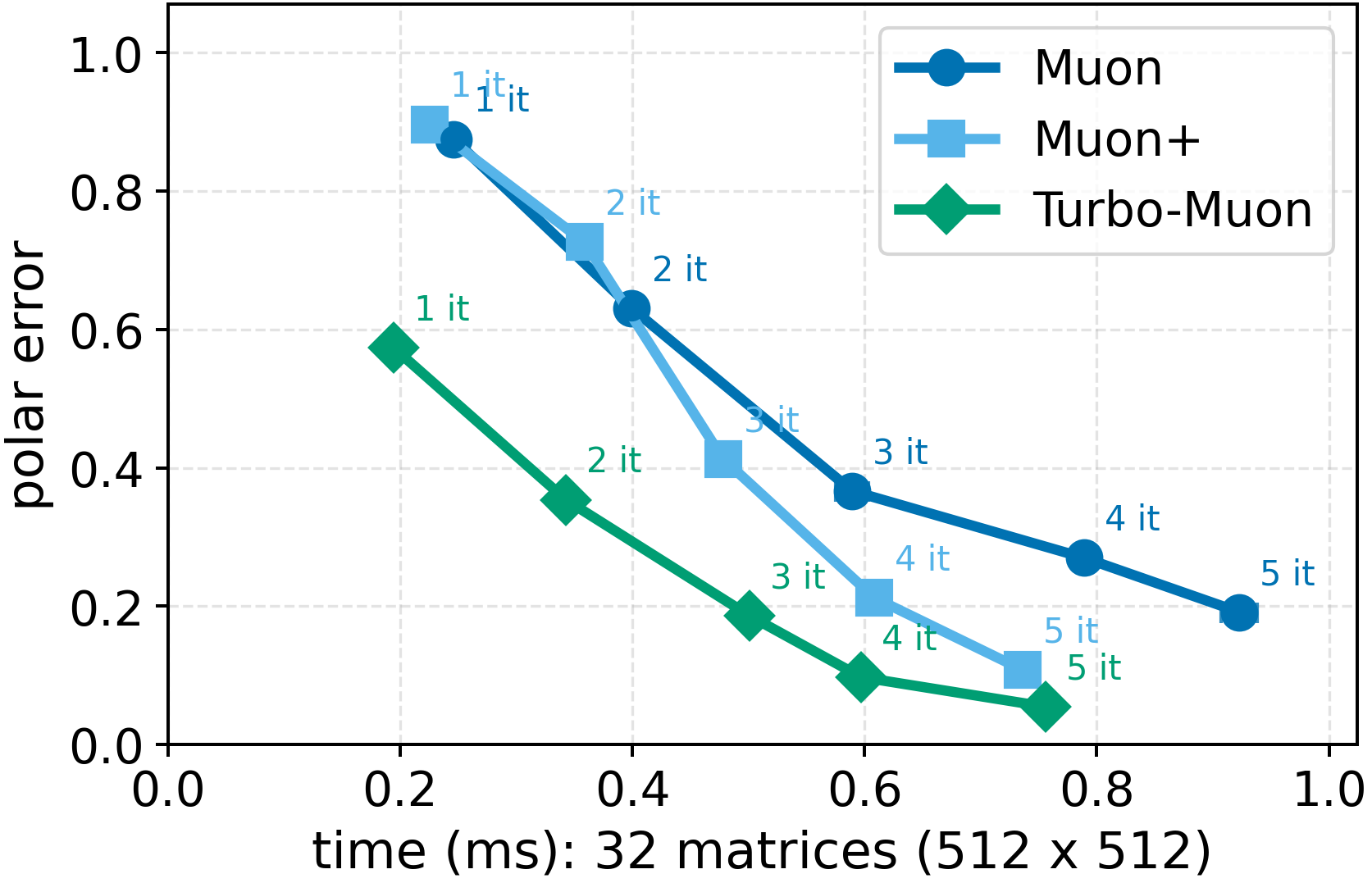}
        \hfill
        
        \includegraphics[width=0.32\textwidth]{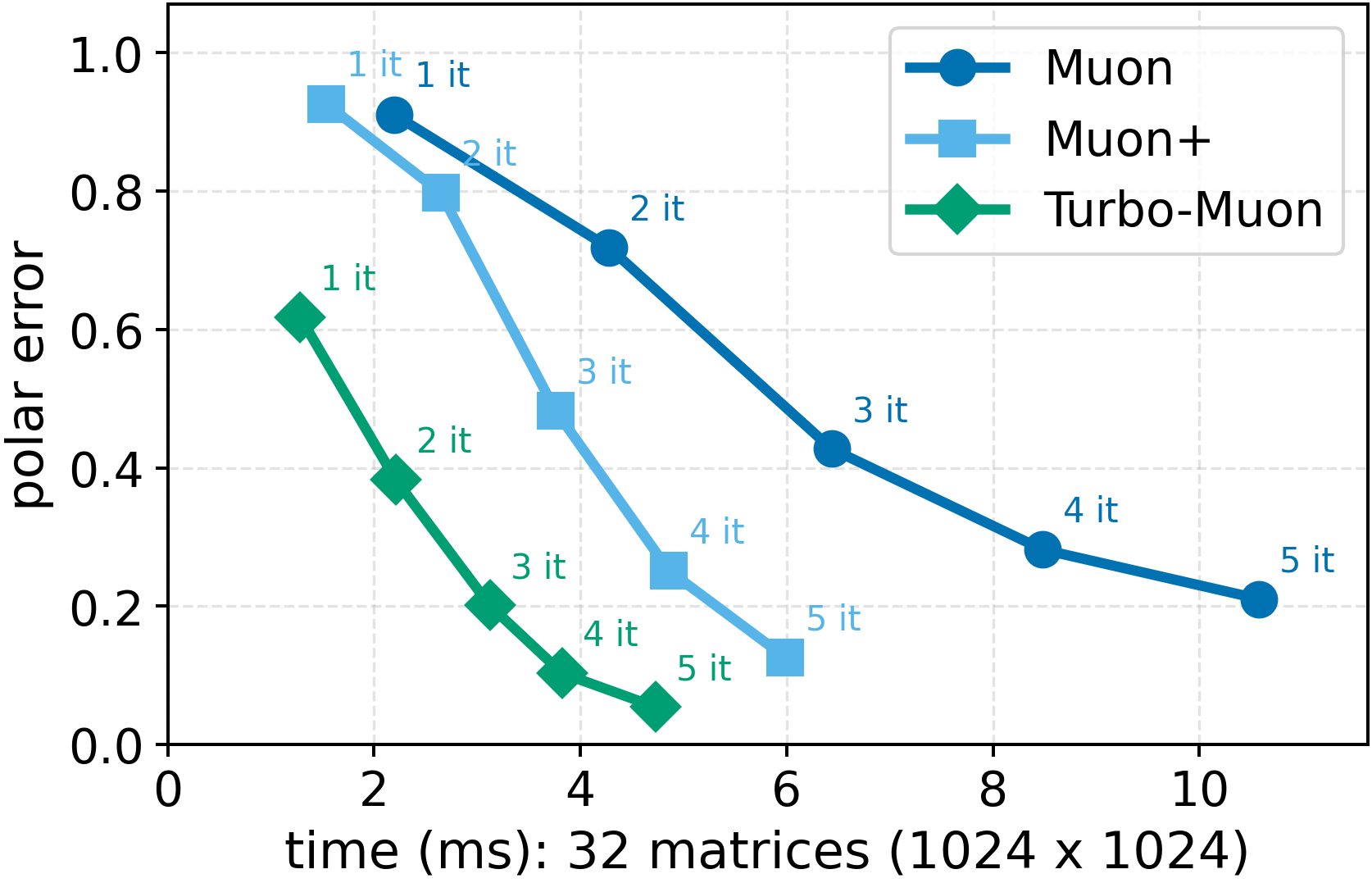}
        \hfill
        \includegraphics[width=0.32\textwidth]{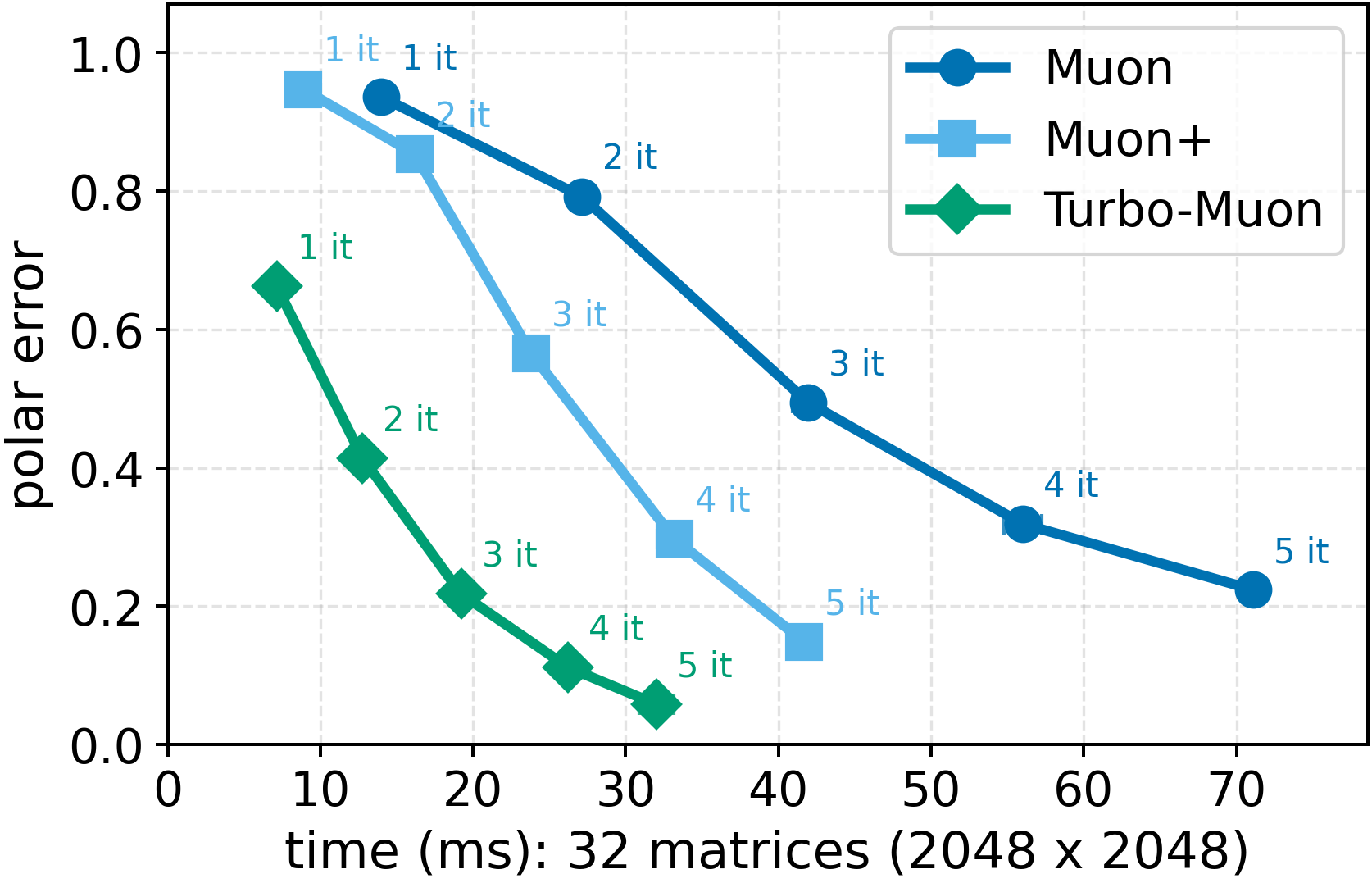}
        \hfill
        \includegraphics[width=0.32\textwidth]{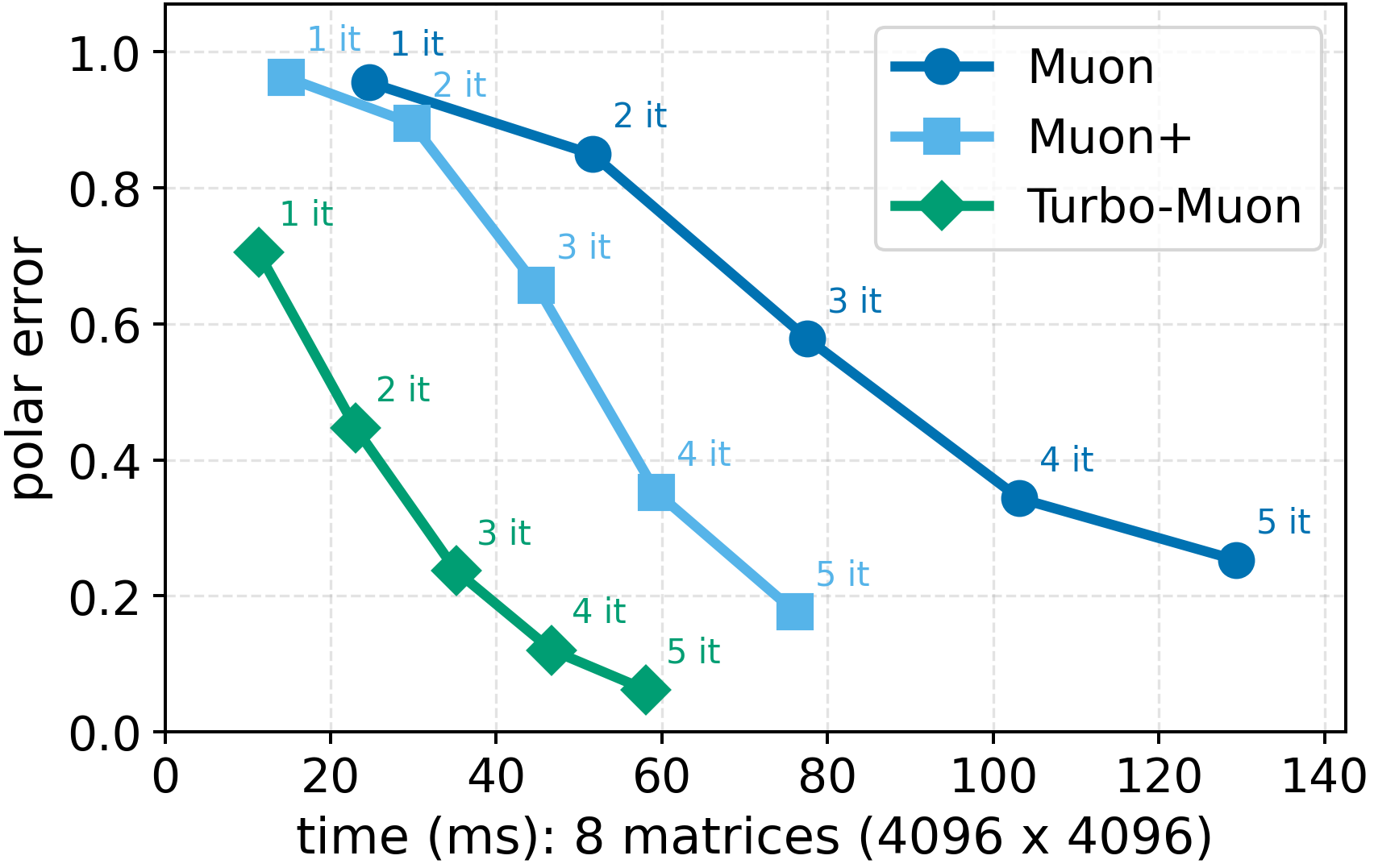}
        
        \caption{Reproduction of \cref{fig:pareto8192} for various matrix sizes. In all settings, \turbomuon outperforms Muon and Muon+ by a significant margin.}
        \label{fig:pareto_all}
    \end{figure*}

\section{Comparison with other pre-conditioning strategies}
\label{ap:other_preconditioning}

AOL preconditioning is one possible way to improve the initialization of the Newton-Schulz iterations, but other normalization strategies may also be considered. As an additional baseline, we evaluate Schatten-$4$ normalization\cite{grishina_accelerating_2025}, which rescales a matrix $X$ by its Schatten-$4$ norm,
\[
    \|X\|_{S_4}
    =
    \left( \sum_i \sigma_i^4 \right)^{1/4}.
\]
This pre-conditioning can also be computed efficiently using the Frobenius norm of $A_0$ (\cref{alg:turbomuon} line~1). Unlike the Frobenius norm, this normalization gives more weight to larger singular values and can therefore provide a stronger control of the spectrum before applying Newton-Schulz.

The results are reported in \cref{fig:schatten4_size,fig:schatten4_iterations}. Schatten-$4$ normalization consistently improves over the standard Frobenius initialization used in Muon. However, AOL preconditioning remains better across matrix sizes and across Newton-Schulz iterations. In particular, AOL reaches lower polar error at each iteration, which is the relevant regime for Muon-style optimizers, where only a small number of Newton--Schulz steps is used in practice.

    \begin{figure*}
        \centering
        \begin{subfigure}{0.49\textwidth}
        \includegraphics[width=\textwidth]{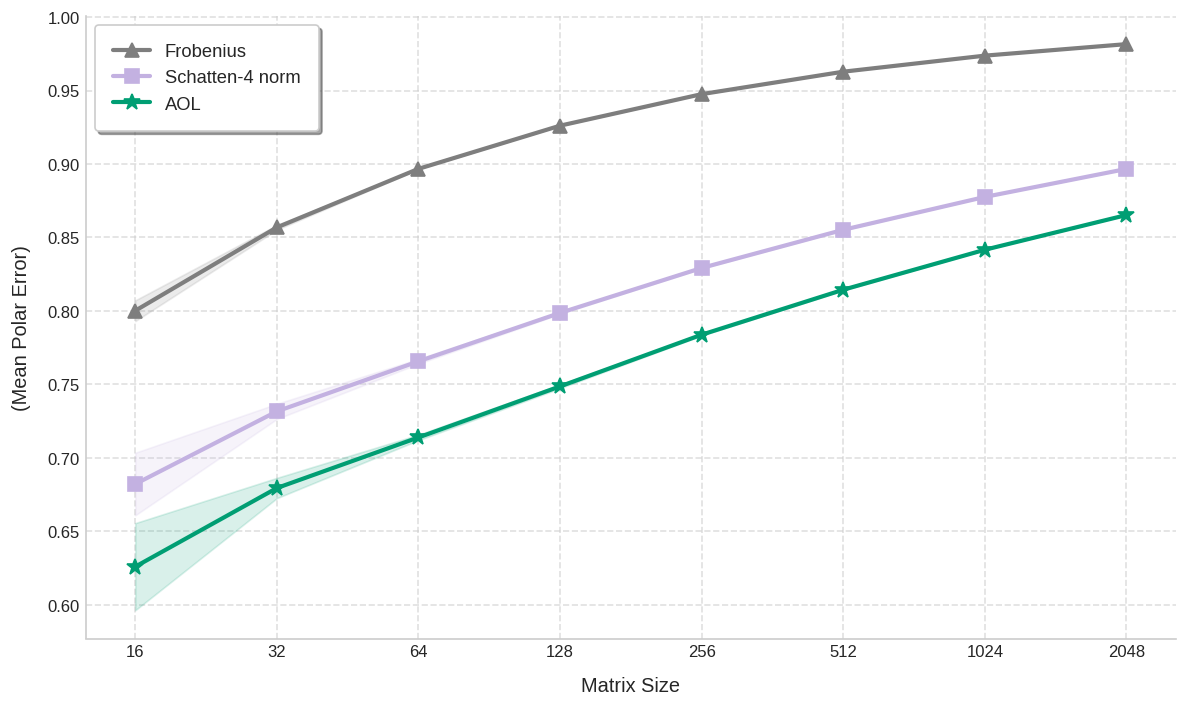}
        \caption{Comparison of several pre-conditioning initializations}
        \label{fig:schatten4_size}
        \end{subfigure}
        \hfill
        \begin{subfigure}{0.49\textwidth}
        \includegraphics[width=\textwidth]{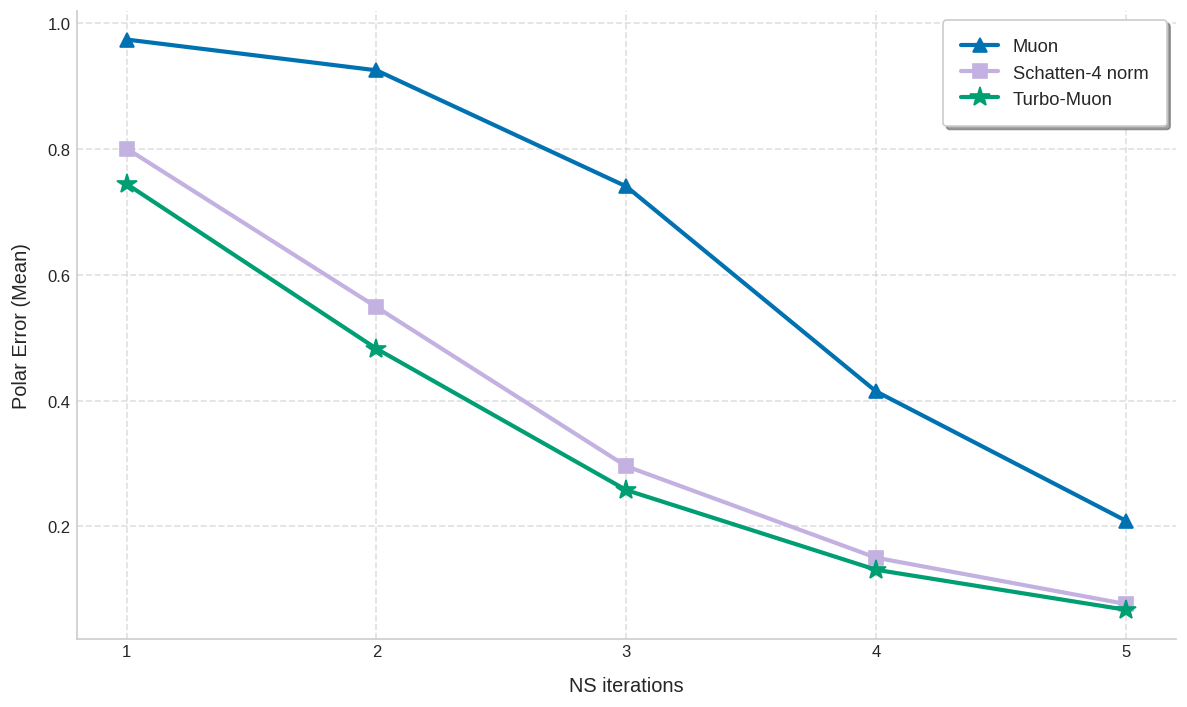}
        \caption{Comparison of Newton-Schulz across iterations with several pre-conditioning initialization}
        \label{fig:schatten4_iterations}
        \end{subfigure}
    \end{figure*}


\section{About the tuning of Newton-Schulz coefficients}
    \label{ap:polar_params}
    
    In this paper, we use dynamic polynomial coefficients to run Newton-Schulz iterations, which means that each step uses different values for $a, b, c$. However, the implementation of Muon+~\cite{ahn_dion_2025} uses coefficients computed explicitly for a fixed number of 5 iterations. To adapt the algorithm to a lower number of iterations, we employed a straightforward strategy as we reduced the number of iterations by retaining only the $n$ last polynomial coefficients. The motivation behind this choice was to isolate the effect of AOL preconditioning as the only factor responsible for the reduction in polar error between Muon+ and \turbomuon. 
    In this section, we demonstrate, through additional experiments, that our strategy is indeed highly effective in regimes with a low number of iterations.
    
    \subsection{Ablation }
    
    The difference between Muon and Muon+ shows the undeniable impact of using explicitly computed polynomial coefficients for a fixed number of steps. However, the improvement between the coefficients introduced by \cite{cesista2025muonoptcoeffs} and the coefficients introduced by \cite{amsel2025polar, grishina_accelerating_2025} is more nuanced, and is mainly observed on heavy-tailed distributions. Yet, the application of those polynomial coefficients to \turbomuon yields a negative impact. We hypothesize the root cause to be due to the assumption about the minimum singular value, required to compute the coefficients. Therefore, we leave the computation of optimal coefficients for \turbomuon as potential future work.
    
    \begin{figure*}[htbp]
        \centering
        \begin{subfigure}{0.32\textwidth}
            \includegraphics[width=\textwidth]{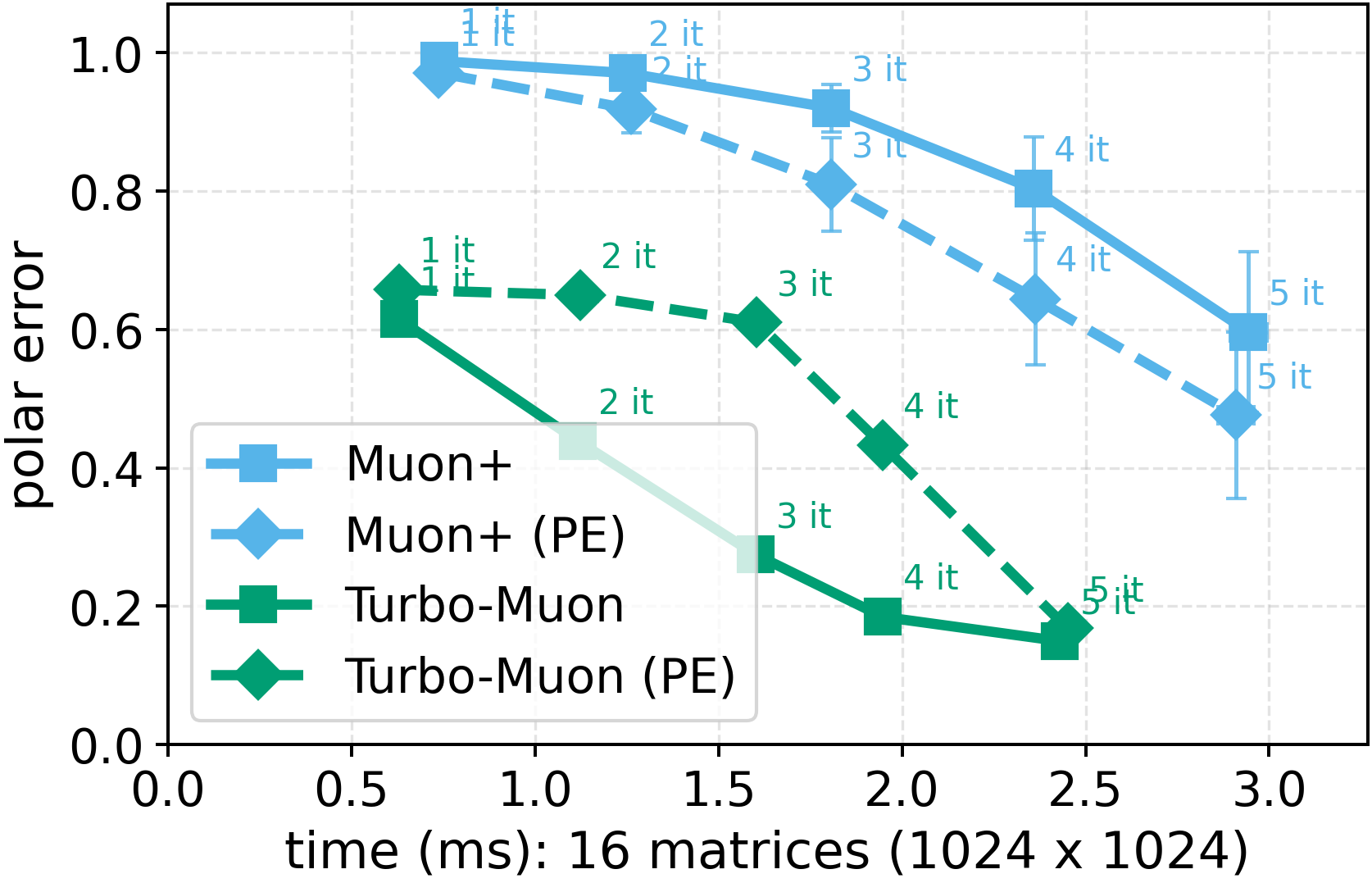}
            \caption{Levy distribution: $\alpha = 1.0$, $\beta = 0$}
        \end{subfigure}
        \hfill
        \begin{subfigure}{0.32\textwidth}
            \includegraphics[width=\textwidth]{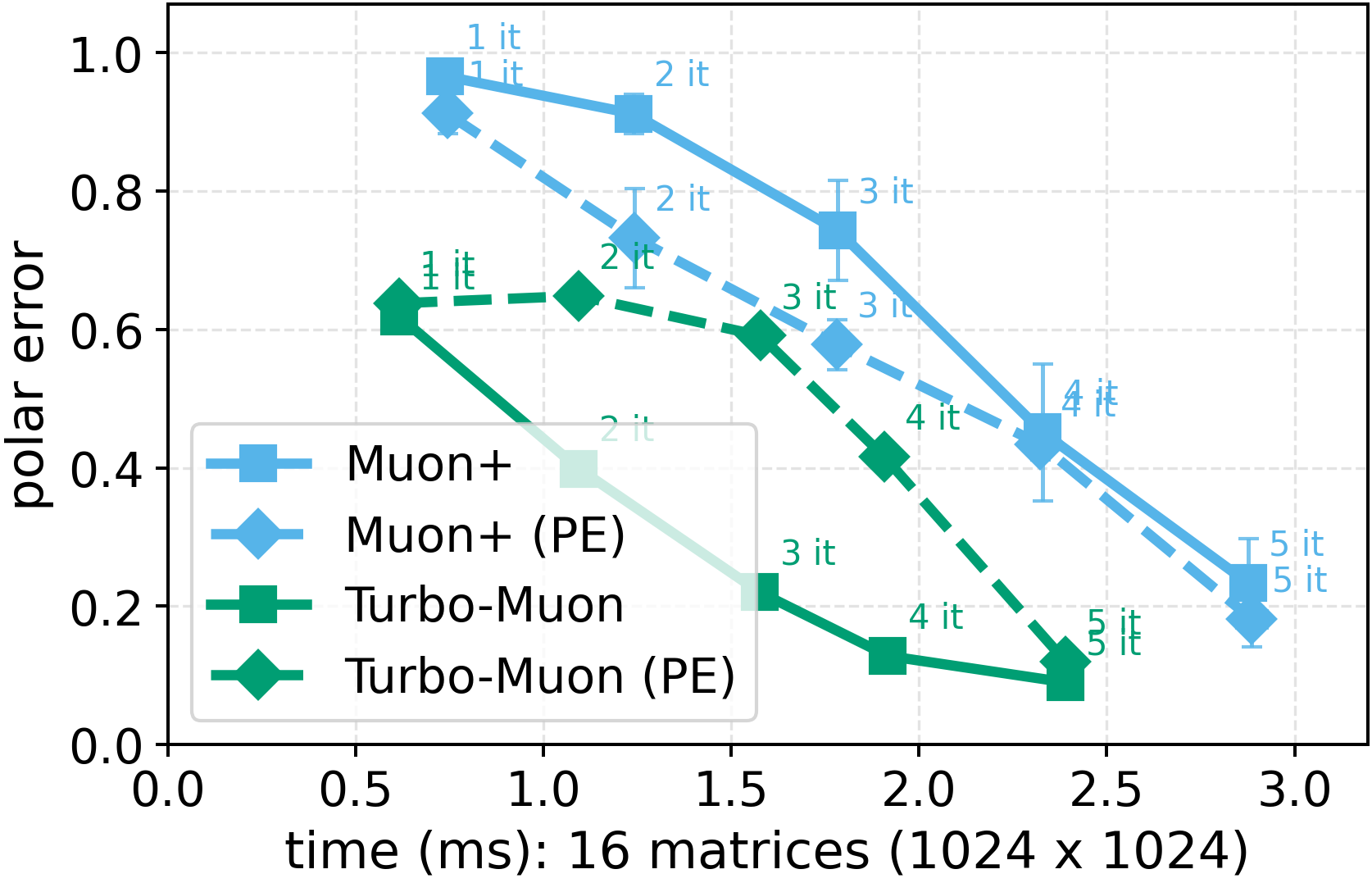}
            \caption{Levy distribution: $\alpha = 1.5$, $\beta = 0$}
        \end{subfigure}
        \hfill
        \begin{subfigure}{0.32\textwidth}
            \includegraphics[width=\textwidth]{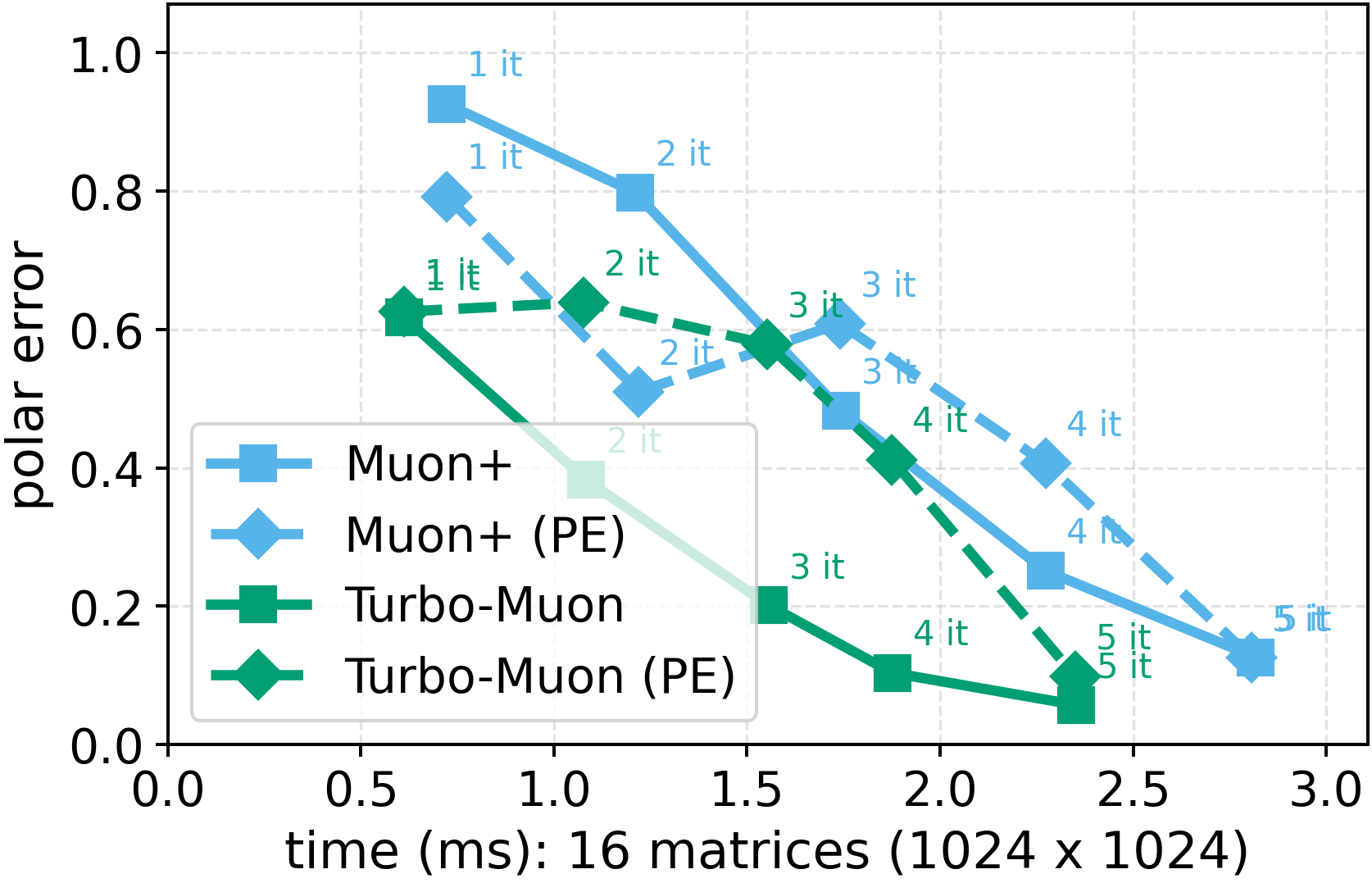}
            \caption{Levy distribution: $\alpha = 2.0$, $\beta = 0$}
        \end{subfigure}
        
        \caption{Coefficients from \cite{amsel2025polar} are suboptimal when combined with AOL-preconditioning. We recomputed the optimal coefficient for each number of iterations, using default parameters ($l=10^{-3}$, $\text{cushion}=0.024$, and $\text{safety\_factor}=2\times10^{-2}$). For both Muon+ and \turbomuon, we report the performance using Polar Express (PE) and compare it with the original. On heavy-tailed distributions, these coefficients improve the performance of Muon+ but degrade the performance of \turbomuon.}
        \label{fig:polar_epxress_ablation}
    \end{figure*}

\section{NanoGPT Speed-Run Details}
\label{ap:nanogpt-WR}

We adopt the Modded-NanoGPT benchmark~\cite{modded_nanogpt_2024}, a highly optimized training setup for a 124M-parameter GPT model on the FineWeb dataset~\cite{penedo2024fineweb}. The goal is to reach a validation cross-entropy of 3.28, corresponding to the estimated performance of GPT-2~\cite{radford2019language}, in minimal wall-clock time. The reference performance is in less than 3 minutes on 8×NVIDIA H100, consuming roughly 0.73B tokens.
This benchmark combines multiple architectural, optimization, and systems-level improvements, including rotary embeddings, QK normalization, ReLU\textsuperscript{2} activations, sliding-window and flash attention, an FP8 language modeling head, and the Muon optimizer. These optimizations are specifically designed to minimize training time and reduce optimizer overhead, making this setting particularly challenging for methods targeting further efficiency gains. Additionally, it places a strong emphasis on reproducible runtime measurement, enabling us to precisely evaluate whether our approach can also yield performance gains in end-to-end training even in highly efficient settings. 
\paragraph{Experimental protocol.} We reuse the most recent public training recipe and modify only the optimizer by replacing Muon with \turbomuon. The number of Newton--Schulz iterations is treated as a configurable parameter, while all other hyperparameters (learning rate, batch size, architecture, and data pipeline) are kept unchanged.
All experiments are conducted on 4$\times$ NVIDIA H100 GPUs. For each configuration, we report the average validation loss and runtime over 10 independent runs to account for runtime variability.
\paragraph{Results.}
As reported in \cref{tab:nanoGPT_accuracy}, \turbomuon achieves equivalent validation performance while requiring one fewer Newton--Schulz iteration. This reduction translates into a consistent decrease in total runtime. In particular, \cref{tab:end2endperf} shows a reduction from $274.75 \pm 0.14$ seconds for the fastest Muon+ configuration to $266.00 \pm 0.10$ seconds when using \turbomuon.



\paragraph{Additional analyses.}
    While the experiment depicted in \cref{tab:nanoGPT_accuracy} compared Muon, Muon+, and \turbomuon, we noted that the nano-GPT speedrun script used the implementation of Muon+ with a different set of NS coefficients from \cite{amsel2025polar, grishina_accelerating_2025} (Depicted as Muon+ (PE) in \cref{ap:polar_params}). Given the additional insights of the experiments of \cref{ap:polar_params}, this choice could be seen as biased in favor of Turbo-Muon. To alleviate any doubts, we reproduced this experiment using the Polar Express factors tested in \cref{ap:polar_params}, which would favor the baseline over \turbomuon.

    In this setting, we compare \turbomuon (PE) with \textbf{four} iterations with Muon+ (PE) with \textbf{five} iterations. In order to remove one iteration, we applied the same strategy that was used between Muon+ and \turbomuon{}, which is discussed in \cref{ap:polar_params}. In order to run this training on 4xH100 instead of 8, we used a gradient accumulation of 2 to match the original results.
    We performed 10 trials in a row in order to measure variance on the final validation loss and the total runtime. This was done for both the baseline and our improved version to obtain comparable results. 

    \begin{table*}[htbp]
        \centering
        \begin{tabular}{llrrrr}
        \toprule
         method & iterations & val loss  & val loss std & runtime & runtime std \\
        \midrule
        Muon+ (PE) & 5it & 3.2774 & 0.0014 & 273.75s & 0.14 \\
        \turbomuon (PE) & 4it & 3.2791 & 0.0013 & 266.65s & 0.10 \\
        \bottomrule
        \end{tabular}
        \caption{Direct comparison against the original training procedure. Results reported for 10 trials for each variant.}
        \label{tab:wr_comp}
    \end{table*}

    Interestingly, we can observe a degradation of only $0.0017$ in the final loss. These results can be observed through the lenses of observations made in \cref{ap:polar_params}: where we observed that factors obtained from \cite{amsel2025polar, grishina_accelerating_2025} underperformed when combined with AOL preconditioning. However, this degradation is smaller than the degradation observed when removing one iteration as done in \cref{tab:nanoGPT_accuracy}.

\section{Runtime influence of the training regime}
\label{ap:trainingregimes}
In practice, the optimizer step accounts for only a small fraction of total training computational cost, which is usually dominated by forward and backward propagation. Nonetheless, Muon’s orthogonalization introduces a measurable overhead compared to AdamW. This overhead decreases as the batch size increases, since Muon’s computational cost is independent of the batch dimension, while the model’s forward cost scales linearly with it. Larger batches, therefore, improve time-to-target performance and typically move Muon past AdamW on the compute-efficiency frontier \cite{shah2025practical}, especially as Muon exhibits a larger critical batch size\cite{wen_fantastic_2025}.

\begin{table*}[t]
        \centering
        \begin{tabular}{lcc}
        \hline
        \textbf{Variant} & \textbf{ms/step} & \textbf{Speedup} \\
        \hline
        \multicolumn{3}{c}{\textbf{Muon variants}} \\
        \hline
        \color{muoncolor}Muon   & $1378.54 \pm 12.35$ & 1.000x \\
        \color{dioncolor}Muon+   & $1308.09 \pm 8.36$  & 1.054x \\
        \color{turbomuoncolor}\turbomuon   & $1261.93 \pm 6.89$  & 1.092x \\
        \hline
        \multicolumn{3}{c}{\textbf{Dion (Muon-derived, large-scale oriented)}} \\
        \hline
        dion-1          & $3477.93 \pm 57.22$ & 0.397x \\
        dion-1/4        & $1678.71 \pm 14.64$ & 0.821x \\
        dion-1/16       & $1229.01 \pm 14.55$ & 1.122x \\
        \hline
        \end{tabular}
        \vspace{1mm}
        \caption{
        \textbf{Turbo-Muon unlocks step-time speedups at medium scale.}
        We report training throughput (ms per optimization step) when simulating the training of a 1.3B-parameter model on an A100 80GB GPU with a batch size of 32k tokens per GPU, a setting where Muon's overhead is clearly exposed, and Dion's low-rank induces a perceptible loss degradation.
        The speedup column indicates performance relative to \texttt{Muon}, with higher values corresponding to faster training.
        }
        \label{tab:1b_runtime}
\end{table*}

\begin{figure*}[t]
    \centering
\end{figure*}

At large model scales, however, global batch size cannot grow indefinitely without sharding, which significantly affects compute efficiency. Naïve replication can increase per-step flops by up to 5×, while more advanced strategies such as sharded matrix multiplication or layer sharding can suffer from bandwidth saturation. The resulting trade-off is tightly linked to the per-device batch size, which limits the amortization of Muon’s overhead \cite{essential2025muon}, leading to an overhead typically around 10-20\%.

Within this context, our preconditioned Newton–Schulz method reduces the cost of the orthogonalization subroutine by approximately 3× at fixed polar accuracy. The resulting end-to-end gains are most pronounced at medium scales—where batch size is reduced but full sharding is unnecessary. This regime is illustrated in \cref{tab:1b_runtime}: On a 1.3B language model, one A100 with 80G memory can achieve a batch size of 32k tokens without sharding. In this context, Muon’s cost is non-negligible, but optimizers tailored for extreme scales, such as Dion\cite{ahn_dion_2025}, are not yet efficient: Dion's low-rank updates--which shine at larger scales--still degrade convergence. Our drop-in preconditioned NS offers a competitive trade-off by reducing computational cost without compromising accuracy.

\section{CIFAR-10 Speed-Run Details}
\label{ap:cifar}

\paragraph{Benchmark description.}
We consider the CIFAR-10 speed-running task~\cite{Jordan2024cifar10airbench}, which aims to reach 94\% validation accuracy on CIFAR-10~\cite{krizhevsky2009learning} in minimal wall-clock time on a single H100 GPU. Unlike NanoGPT, this benchmark involves training a convolutional neural network (CNN).

\paragraph{Optimizer behavior.}
In this setting, Muon reshapes convolutional gradient updates into matrices and applies iterative orthogonalization, providing a complementary regime to evaluate AOL preconditioning.

\paragraph{Results.}
\turbomuon achieves identical accuracy while reducing runtime from $1.359$ s to $1.348$ s, corresponding to a 0.81\% speedup (\cref{tab:end2endperf}).

\paragraph{Discussion.}
Although the absolute gain is modest due to the already optimized implementation, this result confirms that AOL preconditioning consistently improves efficiency across heterogeneous architectures and training setups.

\section{
On the complementarity of  \turbomuon with other Muon improvements
} \label{ap:tridao}

Concurrent with this work, \citet{GramNewtonSchulz} proposed a Gram-based, hardware-aware variant of Newton-Schulz orthogonalization that further reduces the cost of Muon updates. This approach is complementary to our preconditioning strategy, as it modifies the implementation of the Newton-Schulz iterations whereas AOL modifies their initialization. We therefore evaluate their combination in \cref{tab:tridao} evaluating the runtime and residual polar error after five iterations of Newton-Schulz, on random input matrices of size $2048\time 5464$. Gram Newton-Schulz substantially reduces runtime compared to the Muon+ implementation, but slightly increases the polar error. Adding AOL preconditioning recovers part of this loss in approximation quality, while incurring only a modest runtime overhead.

\begin{table}[h]
\centering
\caption{Runtime and polar error of Newton--Schulz variants.}
\label{tab:tridao}
\begin{tabular}{lcc}
\toprule
Variant & Time (ms) & Polar error \\
\midrule
\textcolor{dioncolor}{Muon+}      & 26.031 & $1.7419{\times}10^{-2}$ \\
Gram NS \cite{GramNewtonSchulz}          & 20.621 & $1.7438{\times}10^{-2}$ \\
\textcolor{turbomuoncolor}{Turbo-Muon} + Gram NS    & 22.717 & $1.7431{\times}10^{-2}$ \\
\bottomrule
\end{tabular}
\end{table}

\section{
Beyond Iterative Approximations: On the Asymptotic Behavior of \turbomuon
} \label{ap:bias_analysis}

\begin{figure}
    \centering
    \includegraphics[width=0.45\linewidth]{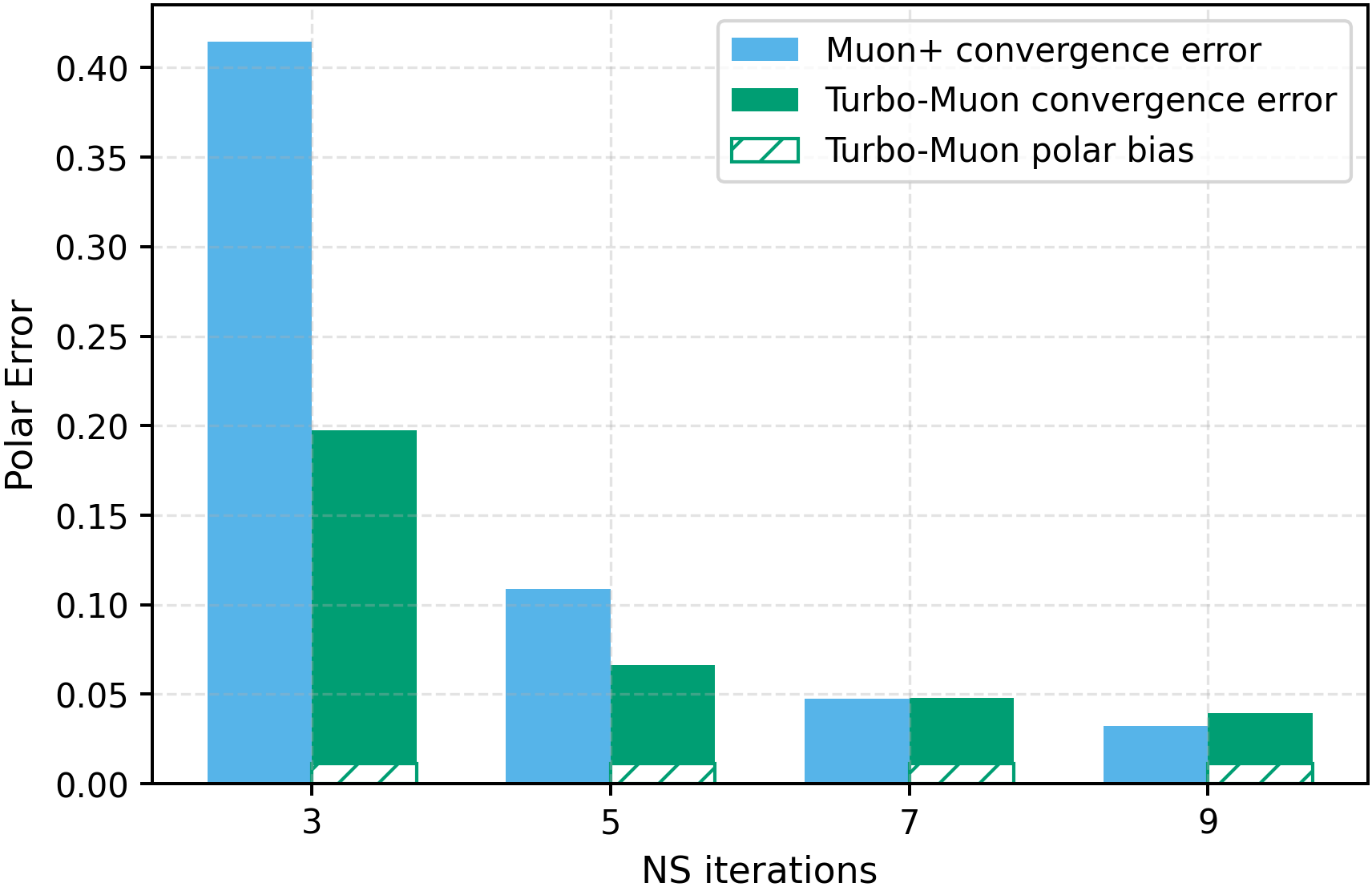}
    \caption{\textbf{Understanding the nature of the remaining polar error.} We decompose the polar error of the \turbomuon algorithm as an approximation error that depends on the number of NS iterations, along with a bias error that is introduced by AOL preconditioning. Results are measured on 100 random normal matrices.}
    \label{fig:bias_converge}
\end{figure}

In the previous sections, we have demonstrated that
$\varepsilon_{\text{polar}}(\operatorname{NS}_t\circ\textrm{AOL}, \mathbf{X}) \lessapprox \varepsilon_{\text{polar}}(\operatorname{NS}_t, \mathbf{X})$
in the widely adopted practical settings where $t \leq 5$. This means that the polar error is of the same order as the usual error made by approximation algorithms. 
However, the nature of this error remains to be determined: When does the remaining polar error come from a lack of orthogonality? or when does it come from a discrepancy between $\textrm{PolarFactor}(\textrm{AOL}(\mathbf{X}))$ and $\textrm{PolarFactor}(\mathbf{X})$ ? 
In the latter case, can this new type of error bias the global optimization process?
In this section, we quantify how AOL preconditioning can introduce a bias and prove that this eventual bias still yields a strict descent direction.

\paragraph{Characterizing the remaining polar error.}
We can quantify the bias error induced by AOL preconditioning, by computing precisely $\operatorname{PolarFactor}(X) = Q$ and $\operatorname{PolarFactor}(AOL(X)) = Q_{aol}$:
\[
    \varepsilon_{\text{bias}}(\mathrm{AOL}, \mathbf{X}) = \frac{\| \textbf{Q} - \textbf{Q}_{aol} \|_F}{\sqrt{n}}
\]
This can be interpreted as the irreducible error that cannot be reduced after full convergence of the Newton-Schulz algorithm with AOL preconditioning, since it converges toward $Q_{aol}$.
Similarly, we can define:
\[
    \varepsilon_{\text{approx}}(\operatorname{NS}_t, \textrm{AOL}(\mathbf{X})) = \frac{\|  \textbf{Q}_{aol} - \operatorname{NS}_t\circ\textrm{AOL}(\mathbf{X}) \|_F}{\sqrt{n}}
\]
Which represents the approximation error between $NS_t(AOL(X))$ and $Q_{aol}$.
These two definitions connect to the polar error with a direct application of the triangle inequality:
\[
    \varepsilon_{\text{polar}}(\operatorname{NS}_t\circ\textrm{AOL}, \mathbf{X}) \leq  
    \underbrace{\varepsilon_{\text{bias}}(\textrm{AOL}, \mathbf{X})}_{\text{irreducible bias}} + 
    \underbrace{\varepsilon_{\text{approx}}(\operatorname{NS}_t, \textrm{AOL}(\mathbf{X}))}_{\text{approximation error}}
\]

This allows us to view the approximation error of the polar factor of X by \turbomuon as a convergence-related error $\varepsilon_\text{approx}$ and a bias-related error $\varepsilon_\text{bias}$ term, which are dependent on the original matrix $\mathbf{X}$. Importantly, we assume that $\lim_{t\to\infty} \varepsilon_\text{approx}(\operatorname{NS}_t, \mathbf{X}) = 0$, as proven in ~\cite{bjorck1971iterative}. To illustrate this phenomenon, in \cref{fig:bias_converge}, we run an experiment where we orthogonalize the same batch of 100 random matrices using Newton-Schulz iterative steps. Since existing methods had coefficients computed for only five iterations, we recomputed those using the method from~\citet{amsel2025polar}. 
For each number of iterations, we quantify $\varepsilon_{\text{bias}}(\textrm{AOL}, \mathbf{X})$ (in dashed green) alongside $\varepsilon_{\text{approx}}(\operatorname{NS}_t, \textrm{AOL}(\mathbf{X}))$ (in dark green), finally, total polar error is compared with the polar error obtained without preconditioning, assuming that Newton-Schulz is perfectly unbiased (in dark blue).
We observe that while \turbomuon is indeed biased, its bias is compensated by the superior approximation speed it confers in practically adopted low $t$ settings. When $t$ is high, the beneficial effects of preconditioning are absorbed while the bias remains. For reference, nine iterations of Newton-Schulz require 27 matrix multiplications in total.

\paragraph{Characterizing AOL preconditioning when $t \rightarrow \infty$.}
We have shown that the bias effect is negligible in the low iteration number of Newton-Schulz usually used in Muon. But we might want to evaluate the effect it could induce in regimes where a high number of Newton-Schulz iterative steps are used. We show that in fact the biased estimation of the polar factor introduced by AOL preconditioning still ensures that we recover a weight update that corresponds to a steepest descent direction. For reference, as explained in~\cite{bernstein2024oldoptimizernewnorm}, 
the steepest descent update in spectral norm is (informally) defined as:
\vspace{1mm}
\begin{equation}
    \arg \min_{\Delta W} \left[ \langle G, \Delta W \rangle + \frac{\lambda}{2} \| \Delta W \|_2^2 \right] = \operatorname{PolarFactor}(G).
\label{eq:spectral_descent}
\end{equation}
\vspace{1mm}

\noindent with $G$ the gradient and $\lambda$ sharpness. Therefore, given that AOL preconditioning can be written as $\operatorname{AOL}(\mathbf{X}) = \mathbf{X}s$ with $s$ as vector with $s_i > 0$,  using Equation~\ref{eq:aol}. We have that the \turbomuon update solves a steepest descent in an induced norm that is dependent of $s$ (c.f. Definition 1 of \citet{bernstein2024oldoptimizernewnorm}).
Moreover, in practice, $s$ forms a vector of strictly positive entries. Therefore, \textit{the weight updates of \turbomuon yield a strict descent direction}. This means that even the worst case of an input that maximizes $\varepsilon_{\text{bias}}$ cannot lead to a divergence of the training process. More details are provided in Appendix~\ref{supp:descent_proof}.

\begin{table}[t]
    \vspace{1mm}
    \centering
    \label{tab:ns-aol}
    \begin{tabular}{cccccc}
    \toprule
    Method & $t_{\text{NS}}$ & Accuracy (\%) & $\varepsilon_\mathrm{approx}$ & $\varepsilon_\mathrm{bias}$ \\
    \midrule
    \turbomuon & 15 & $94.08 \pm 0.1 \%$ & $5.5 \times 10^{-4}$ & $0.1$ \\
    \turbomuon & 25 & $94.07 \pm 0.1 \%$ & $3.6 \times 10^{-4}$ & $0.1$ \\
    \turbomuon & 40 & $ 94.05 \pm 0.1 \%$ & $ 2.1 \times 10^{-4} $ & $0.1$ \\
    \midrule
    Muon & 15 & $94.06 \pm 0.1 \%$ & $1.5 \times 10^{-2}$ & $0.0$ \\
    Muon & 25 & $ 94.06 \pm 0.1 \%$ & $ 1.1 \times 10^{-2} $ & $0.0$ \\
    Muon & 40 & $ 94.04 \pm 0.1 \%$ & $ 7.4 \times 10^{-3} $ & $0.0$ \\
    \bottomrule
    \end{tabular}
    \caption{\textbf{The measured $\varepsilon_\mathrm{bias}$ does not impact final performance:} We measure  the final accuracy (\%) of the trained classifier on the CIFAR-10 dataset with and without AOL preconditioning after $t_{\text{NS}}$ steps. We log the values of $\varepsilon_\mathrm{approx}$ and $\varepsilon_\mathrm{bias}$ once every 10 train batches and report average values across 80 training runs.}
    \label{tab:asymptotic}
\end{table}

\paragraph{Confirming That the Estimation Bias Does Not Affect Training}
In order to empirically validate the theoretical result asserting that this estimation bias should not alter the training convergence, we investigate the training dynamics for an unconventionally high number of Newton-Schulz iterative steps.
We train a neural network on the CIFAR-10 dataset with and without AOL preconditioning. Across each run, we vary the number of iterative NS steps and compute coefficients for the NS iterations according to the number of chosen steps (using the method from~\cite{amsel2025polar}). The high number of steps we choose here leads to an unrealistically long training time, but ensures that the orthogonalized weight updates are extremely close estimates of $\operatorname{PolarFactor}(\mathbf{X})$ and $\operatorname{PolarFactor}(\operatorname{AOL}(\mathbf{X}))$.
Table~\ref{tab:asymptotic} depicts the results for $15\leq t_\mathrm{NS} \leq 40$, and shows that while both \turbomuon and Muon converge toward their respective iterative objective, the bias induced by the AOL-preconditioning has no impact on the quality of training.

Since we have shown in \cref{sec:steepest} that \turbomuon yields steepest descent
directions in an induced norm based on $\operatorname{PolarFactor}(GS)$, 
training a neural network with \turbomuon's weight updates is valid in theory and in practice. Furthermore, usual practical implementations of orthogonality-based optimization rely on a relatively small amount of iterative steps to retrieve an approximation of the closest orthogonal matrix, where AOL preconditioning really shines.

\end{document}